\documentclass[12pt]{article}

\usepackage[margin=1in]{geometry}
\usepackage{amsthm}
\usepackage{amssymb}
\usepackage{amsmath}
\usepackage{graphicx}
\usepackage{mathdots}
\usepackage{mathtools}
\usepackage{url}
\usepackage{color}
\usepackage{framed}
\usepackage{tikz}
\usepackage[ruled,vlined]{algorithm2e}

\newtheorem{theorem}{Theorem}
\newtheorem{corollary}[theorem]{Corollary}
\newtheorem{lemma}[theorem]{Lemma}
\newtheorem{proposition}[theorem]{Proposition}

\theoremstyle{definition}

\title{Sketch-and-solve approaches to $k$-means clustering\\by semidefinite programming}

\author{
Charles~Clum\footnote{Department of Mathematics, The Ohio State University, Columbus, Ohio, USA}
\qquad
Dustin~G.~Mixon\footnotemark[1] \footnote{Translational Data Analytics Institute, The Ohio State University, Columbus, Ohio, USA}
\qquad
Soledad~Villar\footnote{Department of Applied Mathematics \& Statistics, Johns Hopkins University, Baltimore, Maryland, USA} \footnote{Mathematical Institute for Data Science, Johns Hopkins University, Baltimore, Maryland, USA}
\qquad
Kaiying~Xie\footnote{Department of Electrical and Computer Engineering, The Ohio State University, Columbus, Ohio, USA}
}

\date{}

\begin{document}
\maketitle

\begin{abstract}
We introduce a sketch-and-solve approach to speed up the Peng-Wei semidefinite relaxation of $k$-means clustering. When the data is appropriately separated we identify the $k$-means optimal clustering. Otherwise, our approach provides a high-confidence lower bound on the optimal $k$-means value. This lower bound is \emph{data-driven}; it does not make any assumption on the data nor how it is generated. We provide code and an extensive set of numerical experiments where we use this approach to certify approximate optimality of clustering solutions obtained by k-means++.

\end{abstract}

\section{Introduction}

One of the most fundamental data processing tasks is clustering.
Here, one is given a collection of objects, a notion of similarity between those objects, and a clustering objective that scores any given partition according to how well it clusters similar objects together.
The goal is then to partition the objects in a way that optimizes this clustering objective.
For example, to partition the vertices of a simple graph into two clusters, one might take the clustering objective to be edge cut in the graph complement.
This clustering problem is equivalent to MAX-CUT, which is known to be NP-hard~\cite{Karp:72}.
To (approximately) solve MAX-CUT, one may pass to the Goemans--Williamson semidefinite relaxation~\cite{GoemansW:95} and randomly round to a nearby partition.
For certain random graph models, it is known that this semidefinite relaxation is tight with high probability, meaning the clustering problem is exactly solved before the rounding step~\cite{AbbeBH:15,Bandeira:18}.

The above discussion suggests a workflow to solve clustering problems: relax to a semidefinite program (SDP) and round the solution to a partition if necessary.
Since SDPs can be solved to machine precision in polynomial time~\cite{NesterovN:94}, this has been a worthwhile pursuit for several clustering settings~\cite{AbbeBH:15,AwasthiBCKVW:15,MixonVW:17,IguchiMPV:17,MixonX:21,AbdallaB:22}.
However, the polynomial runtime of semidefinite programming is notoriously slow in practice, making it infeasible to cluster thousands of objects (say).
As an alternative, one may pass to a random subset of the objects, cluster the subset by semidefinite programming, and then infer a partition of the full data set based on proximity to the small clusters.
This \textit{sketch-and-solve} approach was studied in~\cite{MixonX:21,AbdallaB:22} in the context of clustering vertices of a graph into two clusters.
By passing to a small subset of objects, the SDP step is no longer burdensome, and one may prove performance guarantees in terms of planted structure in the data.

In this paper, we develop analogous sketch-and-solve approaches to cluster data in Euclidean space.
Here, the standard clustering objective is the $k$-means objective.
Much like MAX-CUT, the $k$-means problem is NP-hard~\cite{AloiseDHP:09,AwasthiCKS:15}, but one may relax to the Peng--Wei semidefinite relaxation~\cite{PengW:07} and obtain performance guarantees~\cite{AwasthiBCKVW:15,IguchiMPV:17,MixonVW:17,LiLLSW:20}.
By sketching to a random subset of data points, one may solve this relaxation quickly and then cluster the remaining data points according to which cluster mean is closest.
Taking inspiration from~\cite{MixonX:21}, this approach was studied in~\cite{ZhuangCY:22} in the setting of Gaussian mixture models, where exact recovery is possible provided the planted clusters are sufficiently separated and the size of the sketch is appropriately large.
In Section~\ref{Sec.exact recover}, we consider data that are not necessarily drawn from a random model.
As we will see, the sketch-and-solve approach exactly recovers planted clusters provided they are sufficiently separated and the size of the sketch is appropriately large relative to the shape of the clusters.

In practice, it is popular to solve the $k$-means problem using Lloyd's method~\cite{Lloyd:82} with the random initialization afforded by $k$-means++~\cite{VassilvitskiiA:06}.
This algorithm produces a partition that is within an $O(\log k)$ factor of the optimal partition in expectation.
In principle, the semidefinite relaxation could deliver an \textit{a posteriori} approximation certificate for this partition, potentially establishing that the partition is even closer to optimal than guaranteed by the $k$-means++ theory.
In Section~\ref{sec.lower bound}, we describe sketch-and-solve approaches to this process that certify constant-factor approximations for data drawn from Gaussian mixture models whenever the dimension of the space is much larger than the number of clusters (even when the clusters are not separated!).
In addition, as established in Section~\ref{sec:numerics}, these algorithms give improved bounds for most of the real-world data sets considered in the original $k$-means++ paper~\cite{VassilvitskiiA:06}.

\subsection{Summary of contributions}

What follows is a brief description of this paper's contributions:

\begin{itemize}
    \item A sketch-and-solve algorithm for $k$-means clustering that consists of subsampling the input dataset, solving an SDP on the samples, and extrapolating the solution on the samples to the entire dataset. We provide theoretical guarantees when the data points are drawn from separated clusters. We phrase the conditions of optimality in terms of \emph{proximity conditions}, a classical concept within the computer science literature \cite{KumarK:10}. These results make use of work from \cite{LiLLSW:20} that connects the clustering SDP with proximity conditions. (See Section \ref{Sec.exact recover}.)
    
    \item An algorithm that takes a dataset as input and outputs a lower bound on the $k$-means optimal value. This algorithm leverages the sketch-and-solve scheme and can be used to certify approximate optimality of clusters obtained with fast, greedy methods such as $k$-means++. We provide bounds on the tightness of this lower bound when data is sampled from spherical Gaussians. (See Section \ref{sec.lower bound}.)
    
    \item Open source code and an extensive set of numerical experiments showing how tight the lower bounds are on several real-world datasets.\footnote{Our code is available here: \url{https://github.com/Kkylie/Sketch-and-solve_kmeans.git}}
\end{itemize}

\subsection{Related work}

The classical semidefinite programming relaxation for $k$-means clustering was proposed by Peng and Wei in~\cite{PengW:07}.
A few years later, \cite{AwasthiBCKVW:15} proved that the Peng--Wei relaxation is tight (i.e., it recovers the solution to the original NP-hard problem) if the clusters are sampled from the so-called \textit{stochastic ball model}~\cite{NelloreW:15} with sufficiently separated balls. 
The proof is based on the construction of a dual certificate. 
These results were significantly improved by \cite{IguchiMPV:17} and~\cite{LiLLSW:20}, and applied to spectral methods in~\cite{LingS:20}. 
In particular, \cite{LiLLSW:20} shows a connection between the tightness of the SDP and the proximity conditions from theoretical computer science~\cite{KumarK:10}. 
The actual threshold at which the SDP becomes tight is currently unknown, an open conjecture was posed in~\cite{LiLLSW:20}. 
 
The Peng--Wei SDP relaxation has also been studied in the context of clustering mixtures of Gaussians, a classical problem in theoretical computer science. 
Introduced in~\cite{Dasgupta:19}, this problem has been approached with many methods including spectral-like methods~\cite{KannanSV:05}, methods of moments~\cite{HsuK:13, WuY:20}, integer programming~\cite{DavisDW:21}, and it has been studied from an information-theoretic point of view~\cite{DiakonikolasKS:17} under many different settings and assumptions.
The performance of the SDP relaxation for clustering Gaussian mixtures was first studied in~\cite{MixonVW:17}, using proof techniques from~\cite{GuedonV:16}. 
It was later shown that the SDP error decays exponentially on the separation of the Gaussians~\cite{FeiC:18, GiraudV:19}.
Other conic relaxations of $k$-means clustering have been recently studied, for instance~\cite{PrasadH:18, PiccialliSW:22}.
A (quite loose) linear programming relaxation of $k$-means was proposed in~\cite{AwasthiBCKVW:15}, with mostly negative results. 
A significantly better LP relaxation was recently introduced and analyzed in~\cite{DerosaK:22}.

Another related line of work concerns efficiently certifying optimality of solutions to data problems via convex relaxations and dual certificates. 
In~\cite{Bandeira:16}, Bandeira proposes leveraging a dual certificate to efficiently certify optimality of solutions obtained with other (more efficient) methods. 
The goal is to combine the theoretical guarantee from (slow) convex relaxations with solutions provided by (fast) algorithms that may not have theoretical guarantees. 
This idea has been used to provide \textit{a posteriori} optimality certificates in data science problems such as point cloud registration~\cite{YangSC:20}, $k$-means clustering~\cite{IguchiMPV:17}, and synchronization~\cite{RosenCBL:19}. 
However, in order for this method to succeed, the relaxation must be tight.

Sketch-and-solve methods provide a looser guarantee (not optimality, but \textit{approximate} optimality) that can work in broader contexts, including ones where no convex relaxation is involved (for instance, numerical linear algebra algorithms~\cite{Woodruff:14}).
Here, we consider a setting where a data problem is relaxed via convex relaxation.
First, the original dataset is subsampled to a \textit{sketch}, then the convex problem is solved in the sketch, and finally a solution is inferred for the entire dataset.
The approximation guarantees from the convex relaxation combined with regularity assumptions on the data (and how well the sample can represent it) can be used to derive approximation guarantees for the general approach.
These ideas were used in~\cite{MixonX:21, AbdallaB:22} in the context of graph clustering, and in~\cite{ZhuangCY:22} in the context of clustering mixtures of Gaussians.
Here, we provide approximation guarantees for a sketch-and-solve approach for general $k$-means clustering, but we also show that the approach provides an efficient algorithm that computes a high-probability lower bound on the $k$-means objective for any dataset, with no assumptions on the data nor how it is generated.
A similar observation was made by two of the authors in a preprint~\cite{MixonV:17}. 

\subsection{Roadmap}

The following section contains some preliminaries and introduces notation for the remainder of the paper.
Next, Section~\ref{Sec.exact recover} introduces a sketch-and-solve algorithm that determines the optimal clustering provided the clusters are appropriately separated.
Section~\ref{sec.lower bound} then introduces sketch-and-solve algorithms to compute lower bounds on the optimal $k$-means value of a given dataset, and we prove that these lower bounds are nearly sharp for data drawn from Gaussian mixtures.
In Section~\ref{sec:numerics}, we illustrate the quality of these bounds on real-world datasets. 
We discuss opportunities for future work in Section~\ref{sec.discussion}, and our main results are proved in Sections~\ref{sec.proof thm2} and~\ref{sec.proof thm 5 and 6}.

\section{Preliminaries and notation}

We are interested in clustering $n$ points in $\mathbb{R}^d$ into $k$ clusters.
Denoting the index set $[n]:=\{1,\ldots,n\}$, let $\Pi(n,k)$ be the set of partitions of $[n]$ into $k$ nonempty sets, i.e., if $\Gamma\in\Pi(n,k)$, then $|\Gamma|=k$, $\bigsqcup_{S\in\Gamma}S=[n]$, and $|S|>1$ for each $S\in\Gamma$.
Given a tuple $X:=\{x_i\}_{i\in[n]}$ of points in $\mathbb{R}^d$ and a nonempty set $S\subseteq[n]$ of indices, we denote the corresponding centroid by
\[
c_S
:=\frac{1}{|S|}\sum_{i\in S}x_i.
\]
Note that we suppress the dependence of $c_S$ on $X$ for simplicity.
With this notation, the (normalized) $k$-means problem is given by
\[
\text{minimize}
\qquad
\frac{1}{n}\sum_{S\in\Gamma}\sum_{i\in S}\|x_i-c_S\|^2
\qquad
\text{subject to}
\qquad
\Gamma\in\Pi(n,k),
\]
and we denote the value of this program by $\operatorname{IP}(X,k)$.
This problem is trivial when $k=1$, and so we assume $k\geq2$ in the sequel.

Lloyd's algorithm is a popular approach to solve the $k$-means problem in practice.
This algorithm alternates between computing cluster centroids and re-partitioning the data points according to the nearest centroid.
These iterations are inexpensive, costing only $O(kdn)$ operations each, and the algorithm eventually converges to a (possibly sub-optimal) fixed point.
The value $V^{(0)}$ of the $k$-means++ random initialization enjoys the following guarantee of approximate optimality (see Theorem~3.1 in~\cite{VassilvitskiiA:06}):
\begin{equation}
\label{eq.kmean++ lower bound}
\operatorname{IP}(X,k)
\geq \mathbb{E} L,
\qquad
L:=\frac{V^{(0)}}{8(\log k +2)}.
\end{equation}
Since the $k$-means objective monotonically decreases with each iteration of Lloyd's algorithm, the $k$-means++ initialization ensures a $O(\log k)$-competitive solution to the $k$-means problem on average.

As a theory-friendly alternative, one may instead solve the $k$-means problem by relaxing to the Peng--Wei semidefinite program~\cite{PengW:07}.
Let $1_S\in\mathbb{R}^n$ denote the indicator vector of $S\subseteq[n]$, and encode $\Gamma\in\Pi(n,k)$ with the matrix
\[
Z_\Gamma
:=\sum_{S\in\Gamma}\frac{1}{|S|}1_S1_S^\top
\in\mathbb{R}^{n\times n}.
\]
Define $D_X\in\mathbb{R}^{n\times n}$ by $(D_X)_{ij}:=\|x_i-x_j\|^2$.
A straightforward manipulation gives
\[
\sum_{S\in\Gamma}\sum_{i\in S}\|x_i-c_S\|^2
=\frac{1}{2}\sum_{S\in\Gamma}\frac{1}{|S|}\sum_{i\in S}\sum_{j\in S}\|x_i-x_j\|^2
=\frac{1}{2}\operatorname{tr}(D_XZ_\Gamma).
\]
Thus, the $k$-means problem is equivalently given by
\[
\text{minimize}
\qquad
\frac{1}{2n}\operatorname{tr}(D_XZ_\Gamma)
\qquad
\text{subject to}
\qquad
\Gamma\in\Pi(n,k).
\]
Considering the containment
\[
\Big\{Z_\Gamma:\Gamma\in\Pi(n,k)\Big\}
\subseteq \mathcal{Z}(n,k)
:=\Big\{Z\in\mathbb{R}^{n\times n}:Z1=1,~\operatorname{tr}Z=k,~Z\geq0,~Z\succeq0\Big\},
\]
we obtain the (normalized) Peng--Wei semidefinite relaxation~\cite{PengW:07}:
\[
\text{minimize}
\qquad
\frac{1}{2n}\operatorname{tr}(D_XZ)
\qquad
\text{subject to}
\qquad
Z\in\mathcal{Z}(n,k).
\]
We denote the value of this program by $\operatorname{SDP}(X,k)$.
When the relaxation is tight, the minimizer recovers the optimal $k$-means partition, and otherwise, the value delivers a lower bound on the optimal $k$-means value.
For many real-world clustering instances, this SDP is too time consuming to compute in practice.
To resolve this issue, we introduce a few \textit{sketch-and-solve} approaches that work similarly well with far less run time.

\section{Sketch-and-solve clustering}
\label{Sec.exact recover}

In this section, we introduce a sketch-and-solve approach to SDP-based $k$-means clustering that works well provided the clusters are sufficiently separated.
Suppose we run the Peng--Wei SDP on a random subset of the data.
If the original clusters are well separated, then this random subset satisfies a proximity condition from~\cite{LiLLSW:20} with high probability, which in turn implies that the Peng--Wei SDP recovers the desired clusters in this subset.
Then we can partition the full data set according to which of these cluster centroids is closest.
This approach is summarized in Algorithm~\ref{alg.exact}. 
The main result of this section (Theorem~\ref{thm.sketch-and-solve-exact}) is a theoretical guarantee for this algorithm.

We start by introducing the necessary notation to enunciate the proximity condition from~\cite{LiLLSW:20}.
For $X\in(\mathbb{R}^d)^n$ and $\Gamma\in\Pi(n,k)$, and for each $S,T\in\Gamma$ with $S\neq T$, we define
\[
\alpha_{ST}
:=\min_{i\in S}\Big\langle x_i-\frac{c_S+c_T}{2},\frac{c_S-c_T}{\|c_S-c_T\|}\Big\rangle,
\quad
\beta_{ST}
:=\frac{1}{2}\Big(\Big(\frac{1}{|S|}+\frac{1}{|T|}\Big)\sum_{R\in\Gamma}\|X_R\|_{2\to2}^2\Big)^{1/2}.
\]
Here, $X_R$ denotes the $d\times n$ matrix whose $i$th column equals zero unless $i\in R$, in which case the column equals $x_i-c_R$.
Notice that $\alpha_{ST}$ captures how close $S$ is to the hyperplane that bisects the centroids $c_S$ and $c_T$, while $\beta_{ST}$ scales some measure of variance of the entire dataset by the sizes of $S$ and $T$.
Intuitively, $k$-means clustering is easier when clusters are well separated, e.g., when the following quantity is positive:
\[
\operatorname{prox}(X,\Gamma)
:=\min_{\substack{S,T\in\Gamma\\S\neq T}}\Big(\alpha_{ST}-\beta_{ST}\Big).
\]

\begin{proposition}[Theorem~2 in~\cite{LiLLSW:20}]
\label{prop.prox}
For each $X\in(\mathbb{R}^d)^n$, there exists at most one $\Gamma\in\Pi(n,k)$ for which $\operatorname{prox}(X,\Gamma)>0$.
Furthermore, if such $\Gamma$ exists, then $Z_\Gamma$ is the unique minimizer of the Peng--Wei semidefinite relaxation for $X$, and therefore $\Gamma$ is the unique minimizer of the $k$-means problem.
\end{proposition}

Next, we introduce notation necessary to enunciate the main result in this section.
Given $X\in(\mathbb{R}^d)^n$ and $\Gamma\in\Pi(n,k)$, consider the quantities
\[
\Delta:=\min_{\substack{S,T\in\Gamma\\S\neq T}}\|c_S-c_T\|,
\qquad
r:=\max_{S\in\Gamma}\max_{i\in S}\|x_i-c_S\|.
\]
We will consider a \textit{shape parameter} of $(X,\Gamma)$ that is determined by the following quantities:
\[
d, \qquad
\frac{\Delta}{r}, \qquad
\frac{\operatorname{prox}(X,\Gamma)}{r}, \qquad
k, \qquad
\frac{1}{n}\min_{S\in\Gamma}|S|, \qquad
\frac{1}{n}\max_{S\in\Gamma}|S|.
\]
Notice that the first term above is completely determined by $X$, the next two terms are determined by $X$ and $\Gamma$, and the last three terms are completely determined by $\Gamma$.
Since its dependence on $X$ factors through $d$, $\frac{\Delta}{r}$, and $\frac{\operatorname{prox}(X,\Gamma)}{r}$, the shape parameter is invariant to rotation, translation, and dilation.

\begin{algorithm}[t]
\SetAlgoLined
\KwData{Points $X=\{x_i\}_{i\in[n]}\in(\mathbb{R}^d)^n$, number of clusters $k$, Bernoulli rate $p$}
\KwResult{Optimal $k$-means clustering $\Gamma\in\Pi(n,k)$}
Draw $W\subseteq[n]$ according to a Bernoulli process with rate $p$ and put $X':=\{x_i\}_{i\in W}$\\
Define $D_{X'}\in\mathbb{R}^{W\times W}$ by $(D_{X'})_{ij}:=\|x_i-x_j\|^2$ for $i,j\in W$\\
Solve Peng--Wei semidefinite relaxation to find optimal $\Gamma'\in\Pi(|W|,k)$ for $X'$\\
Compute centroids $\{c_{S'}\}_{S'\in \Gamma'}$\\
Output $\Gamma\in\Pi(n,k)$ that partitions $X$ according to closest $c_{S'}$
\caption{Sketch-and-solve algorithm for Peng--Wei semidefinite relaxation
 \label{alg.exact}}
\end{algorithm}

\begin{theorem}
\label{thm.sketch-and-solve-exact}
There exists an explicit shape parameter $C\colon(\mathbb{R}^d)^n\times\Pi(n,k)\to[0,\infty]$ for which the following holds:
\begin{itemize}
\item[(a)]
Suppose $X\in(\mathbb{R}^d)^n$ and $\Gamma\in\Pi(n,k)$ satisfy $\operatorname{prox}(X,\Gamma)>0$ and $r\leq\frac{\Delta}{2}$.
Then $C(X,\Gamma)<\infty$, and for $\epsilon\in(0,\frac{1}{2})$, Algorithm~\ref{alg.exact} exactly recovers $\Gamma$ from $X$ with probability $1-\epsilon$ provided
\[
\mathbb{E}|W|
> C(X,\Gamma)\cdot\log(1/\epsilon).
\]
(Here, the randomness is in the Bernoulli process in Algorithm~\ref{alg.exact}.)
\item[(b)]
$C(X,\Gamma)$ is directly related to $d$, $k$, and $\frac{1}{n}\max_{S\in\Gamma}|S|$, and inversely related to  $\frac{\Delta}{r}$, $\frac{\operatorname{prox}(X,\Gamma)}{r}$, and
$\frac{1}{n}\min_{S\in\Gamma}|S|$.
\end{itemize}
\end{theorem}

The proof idea for Theorem~\ref{thm.sketch-and-solve-exact} is simple: identify conditions under which the sketched data satisfy the proximity condition with probability $1-\epsilon$.
Notably, the complexity of the SDP step of Algorithm~\ref{alg.exact} depends on $|W|$, which in turn scales according to the shape of the data rather than its size.
Indeed, when the clusters are more separated, the shape parameter is smaller, and so the sketch size can be taken to be smaller by Theorem~\ref{thm.sketch-and-solve-exact}.

For the sake of illustration, we test the performance of Algorithm~\ref{alg.exact} on data drawn according to the \textit{stochastic ball model}.
Fix $\{\mu_a\}_{a\in[k]}$ in $\mathbb{R}^d$, and for each $a$, draw $\{g_{a,i}\}_{i\in[m]}$ independently according to some rotation-invariant probability distribution $\mathcal{D}$ supported on the origin-centered unit ball in $\mathbb{R}^d$.
Then we write $X \sim \mathsf{SBM}(\mathcal{D}, \{\mu_a\}_{a\in[k]}, m)$ to denote the data $X = \{\mu_a + g_{a,i}\}_{a\in[k],i\in[m]}$.
In words, we draw $m$ points from each of $k$ unit balls centered at the $\mu_a$'s, meaning $n=km$.
We will focus on the case where $d=2$, $k=2$, and $\mathcal{D}$ is the uniform distribution on the unit ball.

First, we consider the behavior of our sketch-and-solve approach as $n\to\infty$.
(Indeed, we have the luxury of considering such behavior since the performance of our approach does not depend on $n$, but rather on the shape of the data.)
In the limit, one may show that
\[
\Delta=\|\mu_1-\mu_2\|,
\qquad
r=1,
\qquad
\operatorname{prox}(X,\Gamma)=\frac{\Delta-3}{2},
\qquad
\frac{1}{n}\min_{S\in\Gamma}|S|=\frac{1}{n}\max_{S\in\Gamma}|S|=\frac{1}{2}.
\]
By Theorem~\ref{thm.sketch-and-solve-exact}(b), it follows that the shape parameter $C(X,\Gamma)$ is inversely related to $\Delta$, as one might expect.
Figure~\ref{fig.exact recovery SBM}(left) illustrates that Algorithm~\ref{alg.exact} exactly recovers the planted clustering of the \textit{entire balls} provided they are appropriately separated and the sketch size $|W|$ isn't too small.
Overall, the behavior reported in Figure~\ref{fig.exact recovery SBM}(left) qualitatively matches the prediction in Theorem~\ref{thm.sketch-and-solve-exact}.

\begin{figure}[t]
\begin{center}
\includegraphics[width=0.49\textwidth,trim={50 200 50 200},clip]{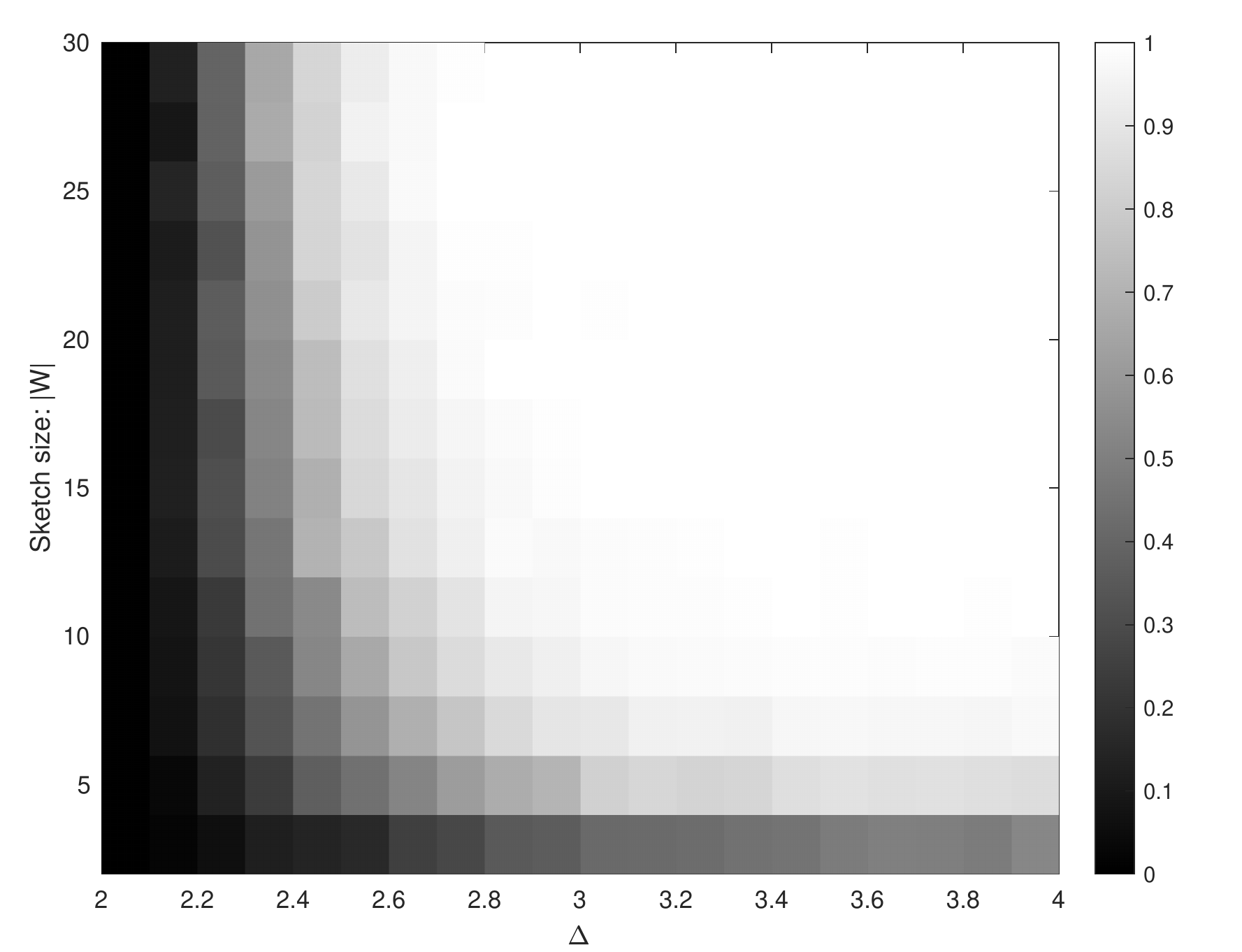}
\includegraphics[width=0.49\textwidth,trim={50 200 50 200},clip]{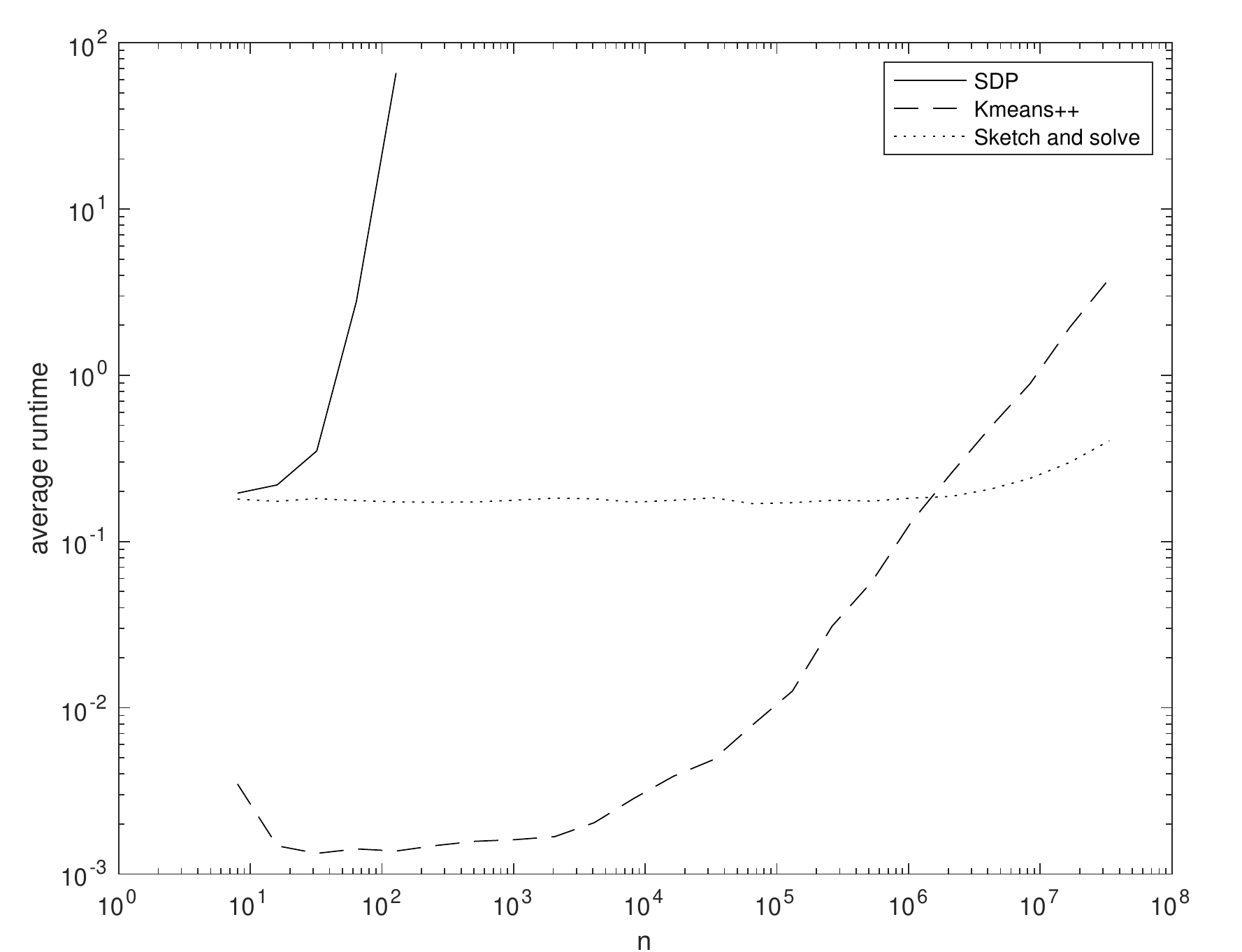}
\end{center}
\caption{\label{fig.exact recovery SBM}
\textbf{(left)}
For each $\Delta \in \{2,2.1,\ldots,4\}$ and $|W| \in \{2,4,\ldots,30\}$, perform the following experiment $500$ times:
Consider two unit balls in $\mathbb{R}^2$ whose centers are separated by $\Delta$.
Draw $|W|$ points uniformly from the union of the balls, solve the Peng--Wei SDP for this sketch in CVX~\cite{GrantB:cvx}, compute the cluster centroids, and partition the union of balls according to the nearest centroid.
Plot the proportion of these $500$ trials for which the planted clustering was exactly recovered.
\textbf{(right)}
For each $n\in\{2^3,2^4,\ldots,2^{25}\}$, perform the following experiment $10$ times:
Draw $n$ points in $\mathbb{R}^2$ according to the uniformly distributed stochastic ball model with separation $\Delta=3$.
Run Algorithm~\ref{alg.exact} with $k=2$ and $p=\min\{\frac{10}{n},1\}$ so that $\mathbb{E}|W| = \min\{10,n\}$.
For comparison, run the Peng--Wei SDP in CVX and MATLAB's built-in $k$-means++ algorithm (both with $k=2$), and plot the average runtimes.
}
\end{figure}

Next, we compare the runtime of our method to traditional (non-sketching) methods in Figure~\ref{fig.exact recovery SBM}(right).
For this, we consider the same stochastic ball model with separation $\Delta=3$.
As mentioned earlier, one cannot run the Peng--Wei SDP directly on a large dataset in a reasonable amount of time.
Also, each iteration of $k$-means++ takes linear runtime.
Meanwhile, the bulk of our sketch-and-solve approach has a runtime that scales with the shape of the data (i.e., it's independent of the size of the data).
Of course, the final clustering step of our approach requires a single pass of the data, which explains the slight increase in runtime for larger datasets.

\section{High-confidence lower bounds}
\label{sec.lower bound}

The previous section focused on a sketch-and-solve approach that recovers the optimal $k$-means clustering whenever the data exhibits a sufficiently nice shape (in some quantifiable sense).
However, the Peng--Wei semidefinite relaxation fails to be tight when the planted clusters are not well separated, so we should not expect exact recovery in general.
Regardless, the SDP always delivers a lower bound on the optimal $k$-means value, which can be useful in practice (e.g., when deciding whether to re-initialize Lloyd's algorithm again to find an even better clustering).

In this section, we introduce sketch-and-solve approaches for computing such bounds in a reasonable amount of time.
Our approaches are inspired by the following simple result:

\begin{lemma}
\label{lem.ip sketch bound}
Consider any sequence $X:=\{x_i\}_{i\in[n]}$ in $\mathbb{R}^d$, draw indices $\{i_j\}_{j\in[s]}$ uniformly from $[n]$ with (or without) replacement, and put $Y:=\{x_{i_j}\}_{j\in[s]}$.
Then
\[
\mathbb{E}\operatorname{SDP}(Y,k)
\leq\mathbb{E}\operatorname{IP}(Y,k)
\leq\operatorname{IP}(X,k).
\]
\end{lemma}

\begin{proof}
The first inequality follows from relaxation.
For the second inequality, select $\Gamma\in\arg\operatorname{IP}(X,k)$, consider the set-valued function $\sigma\colon[n]\to\Gamma$ that satisfies $i\in\sigma(i)$ for every $i\in[n]$, and for each $j\in[s]$, denote the random variable
\[
E_j
:=\bigg\|x_{i_j}-\frac{1}{|\sigma(i_j)|}\sum_{i\in\sigma(i_j)}x_{i}\bigg\|^2.
\]
The random variables $\{E_j\}_{j\in[s]}$ have a common distribution with expectation $\operatorname{IP}(X,k)$, though they are dependent if the indices $\{i_j\}_{j\in[s]}$ are drawn without replacement.
Considering the random function $f\colon[s]\to[n]$ defined by $f(j):=i_j$, we have
\[
\operatorname{IP}(Y,k)
\leq \frac{1}{s}\sum_{S\in\Gamma}\sum_{j\in f^{-1}(S)}\bigg\|x_{i_j}-\frac{1}{|f^{-1}(S)|}\sum_{j'\in f^{-1}(S)}x_{i_{j'}}\bigg\|^2
\leq\frac{1}{s}\sum_{j\in[s]}E_j,
\]
where the second inequality follows from the fact that the centroid of a tuple of points minimizes the sum of squared distances from those points.
The result follows by taking the expectation of both sides.
\end{proof}

By Lemma~\ref{lem.ip sketch bound}, we can lower bound the $k$-means value $\operatorname{IP}(X,k)$ by estimating the expected Peng--Wei value $\mathbb{E}\operatorname{SDP}(Y,k)$ of a random sketch $Y$.
To this end, given an error rate $\epsilon>0$ and $\ell\in\mathbb{N}$ draws of the random sketch, we may leverage concentration inequalities to compute a random variable $B$ that is smaller than $\mathbb{E}\operatorname{SDP}(Y,k)$ with probability $\geq1-\epsilon$.
We provide two such random variables in Algorithms~\ref{alg.hoeffding monte carlo} and~\ref{alg.markov monte carlo}.
Notably, the random variable $B_H$ computed by Algorithm~\ref{alg.hoeffding monte carlo} is consistent in the sense that $B_H$ converges in probability to $\mathbb{E}\operatorname{SDP}(Y,k)$ as $\ell\to\infty$.
Meanwhile, the random variable $B_M$ computed by Algorithm~\ref{alg.markov monte carlo} is \textit{not} consistent, but as we show in Section~\ref{sec:numerics}, it empirically outperforms $B_H$ when $\ell$ is small.
Theorems~\ref{thm.hoefding monte carlo}(a) and~\ref{thm.markov monte carlo}(a) give that these random variables indeed act as lower bounds with probability $\geq1-\epsilon$.

Of course, we would like these lower bounds to be as sharp as possible.
To evaluate how close they are to the desired $k$-means value, we consider the \textit{Gaussian mixture model}.
Given means $\mu_1,\ldots,\mu_k\in\mathbb{R}^d$, covariances $\Sigma_1,\ldots,\Sigma_k\in\mathbb{R}^{d\times d}$, and a probability distribution $p$ over the index set $[k]$, the random vector $x\sim\mathsf{GMM}(\{(\mu_t,\Sigma_t,p(t))\}_{t\in[k]})$ is obtained by first drawing $T$ from $[k]$ with distribution $p$, and then drawing $x$ from the Gaussian $\mathsf{N}(\mu_T,\Sigma_T)$.
The Gaussian mixture model can be thought of as a ``noisy'' version of the stochastic ball model.
By part (b) of the following results, our random lower bounds are nearly sharp provided $d\gg k$, even when there is no separation between the Gaussian means.
See Section~\ref{sec.proof thm 5 and 6} for the proofs of these parts.

\begin{algorithm}[t]
\SetAlgoLined
\KwData{Points $X=\{x_i\}_{i\in[n]}\in(\mathbb{R}^d)^n$, number of clusters $k$}
\KwResult{Indices $\{i_j\}_{j\in[k]}\in[n]^k$ of well-separated points}
Put $i_1:=1$, and iteratively select $i_{t+1}\in\arg\max_{i\in[n]}\min_{j\in[t]}\|x_i-x_{i_j}\|$
\caption{Deterministic $k$-means$++$ initialization
 \label{alg.deterministic k-means++}}
\end{algorithm}

\begin{algorithm}[t]
\SetAlgoLined
\KwData{Points $X=\{x_i\}_{i\in[n]}\in(\mathbb{R}^d)^n$, number of clusters $k$, sketch size $s$, number of trials $\ell$, error rate $\epsilon$}
\KwResult{Random variable $B_H$ such that $\operatorname{IP}(X,k)\geq B_H$ with probability $\geq 1-\epsilon$}
Run deterministic $k$-means$++$ initialization and put $b:=\max_{i\in[n]}\min_{j\in[k]}\|x_i-x_{i_j}\|^2$\\
Draw $\{Y_i\}_{i\in[\ell]}$ independently at random, with each $Y_i$ denoting $s$ points drawn uniformly from $X$ with replacement\\
Output $B_H:=\frac{1}{\ell}\sum_{i\in[\ell]}\operatorname{SDP}(Y_i,k)-(\frac{b^2}{2\ell}\log(\frac{1}{\epsilon}))^{1/2}$
\caption{Hoeffding Monte Carlo $k$-means lower bound
 \label{alg.hoeffding monte carlo}}
\end{algorithm}

\begin{algorithm}[t]
\SetAlgoLined
\KwData{Points $X=\{x_i\}_{i\in[n]}\in(\mathbb{R}^d)^n$, number of clusters $k$, sketch size $s$, number of trials $\ell$, error rate $\epsilon$}
\KwResult{Random variable $B_M$ such that $\operatorname{IP}(X,k)\geq B_M$ with probability $\geq 1-\epsilon$}
Draw $\{Y_i\}_{i\in[\ell]}$ independently at random, with each $Y_i$ denoting $s$ points drawn uniformly from $X$ without replacement\\
Output $B_M:=\epsilon^{1/\ell}\min_{i\in[\ell]}\operatorname{SDP}(Y_i,k)$
\caption{Markov Monte Carlo $k$-means lower bound
 \label{alg.markov monte carlo}}
\end{algorithm}

\begin{theorem}[Performance guarantee for Algorithm~\ref{alg.hoeffding monte carlo}]\
\label{thm.hoefding monte carlo}
\begin{itemize}
\item[(a)]
Consider any $X:=\{x_i\}_{i\in[n]}$ in $\mathbb{R}^d$, any $k,s,\ell\in\mathbb{N}$, and $\epsilon>0$, and compute the random variable $B_H$ in Algorithm~\ref{alg.hoeffding monte carlo}.
Then
\[
\operatorname{IP}(X,k)
\geq B_H
\]
with probability $\geq1-\epsilon$.
(Here, the probability is on Algorithm~\ref{alg.hoeffding monte carlo}.)
\item[(b)]
Consider any $\mu_1,\ldots,\mu_k\in\mathbb{R}^d$, draw the points $X:=\{x_i\}_{i\in[n]}$ independently with distribution $\mathsf{GMM}(\{(\mu_t,I_d,\frac{1}{k})\}_{t\in[k]})$, take any $s,\ell\in\mathbb{N}$, and $\epsilon>0$, and compute the random variable $B_H$ in Algorithm~\ref{alg.hoeffding monte carlo}.
Then
\[
B_H\geq\frac{d-6k-2}{d+1}\cdot \operatorname{IP}(X,k)
\]
with probability $\geq 1-\frac{1}{n}-e^{-\Omega(n/(d+\log n)^2)}-\epsilon$ provided
\[
s\geq15d\log d,
\qquad
\ell\geq128(d+3\log n)^2\log(1/\epsilon).
\]
(Here, the probability is on both $X$ and Algorithm~\ref{alg.hoeffding monte carlo}.)
\end{itemize}
\end{theorem}

\begin{proof}[Proof of Theorem~\ref{thm.hoefding monte carlo}(a)]
Take $\{x_{i_j}\}_{j\in[k]}$ from Algorithm~\ref{alg.deterministic k-means++}, consider any tuple $\{i'_t\}_{t\in[s]}$ of indices in $[n]$, and put $Y:=\{x_{i'_t}\}_{t\in[s]}$.
Then
\begin{align*}
\operatorname{SDP}(Y,k)
&\leq\operatorname{IP}(Y,k)
=\min_{\mu_1,\ldots,\mu_k\in\mathbb{R}^d}\frac{1}{s}\sum_{t\in[s]}\min_{j\in[k]}\|x_{i'_t}-\mu_j\|^2\\
&\leq\frac{1}{s}\sum_{t\in[s]}\min_{j\in[k]}\|x_{i'_t}-x_{i_j}\|^2
\leq\max_{t\in[s]}\min_{j\in[k]}\|x_{i'_t}-x_{i_j}\|^2
\leq\max_{i\in[n]}\min_{j\in[k]}\|x_i-x_{i_j}\|^2.
\end{align*}
It follows that $\operatorname{SDP}(Y_i,k)\leq b$ almost surely for each $i\in[\ell]$.
The result then follows from Lemma~\ref{lem.ip sketch bound} and Hoeffding's inequality:
\[
\mathbb{P}\{B_H>\operatorname{IP}(X,k)\}
\leq \mathbb{P}\bigg\{\frac{1}{\ell}\sum_{i\in[\ell]}\operatorname{SDP}(Y_i,k)-\Big(\tfrac{b^2}{2\ell}\log(\tfrac{1}{\epsilon})\Big)^{1/2}>\mathbb{E}\operatorname{SDP}(Y_1,k)\bigg\}
\leq\epsilon.
\qedhere
\]
\end{proof}

\begin{theorem}[Performance guarantee for Algorithm~\ref{alg.markov monte carlo}]\
\label{thm.markov monte carlo}
\begin{itemize}
\item[(a)]
Consider any $X:=\{x_i\}_{i\in[n]}$ in $\mathbb{R}^d$, any $k,s,\ell\in\mathbb{N}$, and $\epsilon>0$, and compute the random variable $B_M$ in Algorithm~\ref{alg.markov monte carlo}.
Then
\[
\operatorname{IP}(X,k)
\geq B_M
\]
with probability $\geq1-\epsilon$.
(Here, the probability is on Algorithm~\ref{alg.markov monte carlo}.)
\item[(b)]
Consider any $\mu_1,\ldots,\mu_k\in\mathbb{R}^d$, draw the points $X:=\{x_i\}_{i\in[n]}$ independently with distribution $\mathsf{GMM}(\{(\mu_t,I_d,\frac{1}{k})\}_{t\in[k]})$, take any $s,\ell\in\mathbb{N}$, and $\epsilon>0$, and compute the random variable $B_M$ in Algorithm~\ref{alg.markov monte carlo}.
Then
\[
B_M\geq\frac{d-3k-2}{d+1}\cdot \operatorname{IP}(X,k)
\]
with probability
$\geq 1-e^{-s/(8d)}-2e^{-s/54}-e^{-n/(16d)}$ provided
\[
s\geq 54d\log\ell,
\qquad
\ell\geq d\log(1/\epsilon),
\qquad
\ell\geq 3.
\]
(Here, the probability is on both $X$ and Algorithm~\ref{alg.markov monte carlo}.)
\end{itemize}
\end{theorem}

\begin{proof}[Proof of Theorem~\ref{thm.markov monte carlo}(a)]
By Lemma~\ref{lem.ip sketch bound}, we have $\operatorname{IP}(X,k)\geq\mathbb{E}\operatorname{SDP}(Y_1,k)$.
It follows that
\[
\mathbb{P}\{B_M>\operatorname{IP}(X,k)\}
\leq\mathbb{P}\{B_M>\mathbb{E}\operatorname{SDP}(Y_1,k)\}
=\mathbb{P}\{\epsilon^{1/\ell}\operatorname{SDP}(Y_1,k)>\mathbb{E}\operatorname{SDP}(Y_1,k)\}^\ell,
\]
where the last step applies the fact that $\{\operatorname{SDP}(Y_i,k)\}_{i\in[k]}$ are independent with identical distribution.
Markov's inequality further bounds this upper bound by $\epsilon$.
\end{proof}

The proof of Theorem~\ref{thm.hoefding monte carlo}(b) makes use of the following approximation ratio afforded by deterministic $k$-means++ initialization (Algorithm~\ref{alg.deterministic k-means++}), which may be of independent interest.

\begin{lemma}
\label{lem.deterministic kmeans++ approximation ratio}
Given points $X:=\{x_i\}_{i\in[n]}$ in $\mathbb{R}^d$, consider the problem of minimizing the function $f\colon(\mathbb{R}^d)^k\to\mathbb{R}$ defined by
\[
f(\{z_j\}_{j\in[k]})
:=\max_{i\in[n]}\min_{j\in[k]}\|x_i-z_j\|.
\]
The output $\{i_j\}_{j\in[k]}$ of deterministic $k$-means$++$ initialization (Algorithm~\ref{alg.deterministic k-means++}) indexes an approximate solution to this optimization problem with approximation ratio $2$:
\[
f(\{x_{i_j}\}_{j\in[k]})\leq 2\inf(f).
\]
\end{lemma}

\begin{proof}
First, we establish that $f$ has a global minimizer.
Taking $z_j=0$ for every $j$ gives $f(\{z_j\}_{j\in[k]})=\max_{i\in[n]}\|x_i\|=:R$.
Let $w\in\mathbb{R}^d$ denote a point for which $\min_{i\in[n]}\|x_i-w\|=R$.
For any input $\{z_j\}_{j\in[k]}$ such that $f(\{z_j\}_{j\in[k]})\leq R$, then for each $j$, either $z_j$ is within $R$ of the nearest $x_i$, or the function value is not changed by sending $z_j$ to $w$.
Thus, we may restrict the input space to the set of $\{z_j\}_{j\in[k]}$ for which every $z_j$ is within $R$ of the nearest $x_i$.
Since this set is compact and $f$ is continuous, the extreme value theorem guarantees a global minimizer.

Let $\{z_j^\star\}_{j\in[k]}$ denote a global minimizer of $f$, and put $r:=\min(f)$.
Then every $x_i$ is within $r$ of some $z_j^\star$.
Run deterministic $k$-means$++$ initialization, and consider the quantities
\[
r_t:=\max_{i\in[n]}\min_{j\in[t]}\|x_i-x_{i_j}\|
\]
for $t\in[k]$.
Observe that $r_1\geq\cdots\geq r_{k}=f(\{x_{i_j}\}_{j\in[k]})$.
We wish to demonstrate $r_k\leq 2r$.
If $r_{k-1}\leq 2r$, then we are done.
Otherwise, we have $r_t>2r$ for every $t<k$, and so the pairwise distances between $\{x_{i_j}\}_{j\in[k]}$ are all strictly greater than $2r$.
Since each $x_{i_j}$ is within $r$ of a point in $\{z_j^\star\}_{j\in[k]}$, the triangle inequality gives that no two $x_{i_j}$'s are within $r$ of the same point in $\{z_j^\star\}_{j\in[k]}$.
Since every $x_i$ is within $r$ of some $z_j^\star$, which in turn is within $r$ of some $x_{i_j}$, the triangle inequality further gives that every $x_i$ is within $2r$ of some $x_{i_j}$.
This establishes $r_k\leq 2r$, as desired.
\end{proof}

\section{Numerical experiments}
\label{sec:numerics}

In this section, we use both synthetic and real-world data to test the performance of our high-confidence lower bounds $B_H$ (Algorithm~\ref{alg.hoeffding monte carlo}) and $B_M$ (Algorithm~\ref{alg.markov monte carlo}).

\subsection{Implementation details}

When implementing our algorithms, we discovered a few modifications that delivered substantial improvements in practice.

\textbf{Correcting the numerical dual certificate.}
SDPNAL+~\cite{SunYT:15} is a fast SDP solver that was specifically designed for SDPs with entrywise nonnegativity constraints, such as the Peng--Wei SDP.
For this reason, we used SDPNAL+ to solve our sketched SDP, but then it delivered a dual certificate that was not exactly dual feasible.
To be explicit, recall the normalized Peng--Wei SDP for a sketch $Y$ of size $s$:
\[
\text{minimize}
\qquad
\frac{1}{2s}\operatorname{tr}(D_{Y}Z)
\qquad
\text{subject to}
\qquad
Z\in\mathbb{R}^{s\times s},~\mathcal{A}(Z) = b,~Z\geq0,~Z\succeq0,
\]
where $\mathcal{A}\colon\mathbb{R}^{s\times s}\to\mathbb{R}^{s+1}$ and $b \in \mathbb{R}^{s+1}$ are defined by
\begin{align*}
\mathcal{A}(Z) := \left[
\begin{matrix}
\operatorname{tr}(Z)\\
Z1
\end{matrix}
\right],
\qquad
b := \left[
\begin{matrix}
k\\
1
\end{matrix}
\right].
\end{align*}
The corresponding dual program is
\[
\text{minimize}
\quad
\frac{1}{2s} b^\top y
\quad
\text{subject to}
\quad
y\in\mathbb{R}^{s+1},~P,S\in\mathbb{R}^{s\times s},~\mathcal{A}^*(y) + P + S = D_Y,~P\geq0,~S\succeq0.
\]
In practice, SDPNAL+ returns $(y_0,P_0,S_0)$ such that $\mathcal{A}^*(y) + P + S \neq D_Y$.
To correct this, we fix $S=S_0$ in the dual program, thereby restricting to an easy-to-solve linear program, which we then solve using CVX.
The result is a point $(y_1,P_1,S_0)$ that is feasible in the original dual program, and so weak duality implies that $\frac{1}{2s}b^\top y_1$ is a lower bound, as desired.

\textbf{Truncating the distance matrix.}
For some datasets (e.g., INTRUSION), the entries of $D_X$ might vary widely, and we observe that SDP solvers fail to converge in such cases.
This can be resolved by truncating.
Specifically, we apply the function $t\mapsto\min\{t,10^8\}$ to each entry of $D_X$ before solving the SDP.
Not only does this make the problem solvable in practice, the optimal value of the truncated problem is a lower bound on the original optimal value, and so we can use the result.

\textbf{Truncating the sketched SDP value.}
When we run Algorithm~\ref{alg.hoeffding monte carlo} in practice, we find that $b$ is often so large that $B_H$ is negative (and therefore useless).
Recall that $b$ plays the role of an almost sure upper bound on $\operatorname{SDP}(Y,k)$.
We can replicate this behavior by selecting a threshold $u>0$ and replacing $\operatorname{SDP}(Y,k)$ with the truncated version $\min\{\operatorname{SDP}(Y,k),u\}$.
This leads to the following modification:
\[
B_H=\frac{1}{\ell}\sum_{i\in[\ell]} \min \Big\{\operatorname{SDP}(Y_i,k),u\Big\}-\Big(\tfrac{u^2}{2\ell}\log(\tfrac{1}{\epsilon})\Big)^{1/2}.
\]
Following the proof of Theorem~\ref{thm.hoefding monte carlo}(a), we have that
\begin{align*}
&\mathbb{P}\{B_H>\operatorname{IP}(X,k)\} \\
&\quad\leq \mathbb{P}\bigg\{\frac{1}{\ell}\sum_{i\in[\ell]}\min\Big\{\operatorname{SDP}(Y_i,k),u\Big\}-\Big(\tfrac{u^2}{2\ell}\log(\tfrac{1}{\epsilon})\Big)^{1/2}>\mathbb{E}\min\Big\{\operatorname{SDP}(Y_1,k),u\Big\}\bigg\} 
\leq\epsilon,
\end{align*}
and so it still delivers a high-confidence lower bound.
In practice, we take $u$ to be the $k$-means value of the clustering given by $k$-means++, and we observe that the truncation usually occurs for only a few of the $\ell$ sketches.

\textbf{Other implementation details.}
When solving our SDPs with SDPNAL+, we warm start with a block diagonal primal matrix $Z$ that we obtain by running $k$-means++ over the sketched data.
Also, we solve SDPs to low precision by default, but in cases where this only takes a few iterations, we solve the SDP again to a higher precision.

\subsection{Datasets and parameters}
\label{sec.datasets and parameters}
We test our algorithms on five datasets.
\textbf{MNIST}, consisting of $60000$ points in $784$ dimensions, is the training set images of the MNIST database of handwritten digits~\cite{LeCunC:10}.
\textbf{NORM-10}, consisting of $10000$ points in $5$ dimensions, is a synthetic dataset drawn according to a mixture of $10$ Gaussians $\mathsf{GMM}(\{(\mu_t,I_5,\frac{1}{10})\}_{t\in[10]})$ with the centers $\{\mu_t\}_{t\in[10]}$ drawn uniformly in a $5$-dimensional hypercube of side length $500$.
Similarly \textbf{NORM-25} consisting of $10000$ points in $15$ dimensions, is a synthetic dataset drawn according to a mixture of $25$ Gaussians $\mathsf{GMM}(\{(\mu_t,I_{15},\frac{1}{25})\}_{t\in[25]})$ with the centers $\{\mu_t\}_{t\in[25]}$ drawn uniformly in a $15$-dimensional hypercube of side length $500$.
\textbf{CLOUD}, consisting of $1024$ points in $10$ dimensions, is Philippe Collard's cloud cover data~\cite{Collard:CLOUD} in the UC Irvine Machine Learning Repository.
\textbf{INTRUSION}, consisting of $494021$ points in $34$ dimensions, is a $10\%$ subset data of continuous features (with symbolic features removed) for network intrusion detection used in KDD Cup 1999 competition~\cite{KDD:INTRUSION}, also from the UC Irvine Machine Learning Repository.
Notice that NORM-25, CLOUD, and INTRUSION are the same datasets used in~\cite{VassilvitskiiA:06} to show the performance of $k$-means++ algorithm.

For each dataset, we run our algorithms in various settings.
First, we run our algorithms with a small number of trials, namely, $\ell = 30$.
We test $k=10$ for MNIST and $k\in\{10,25,50\}$ for each of the other four datasets.
We expect the Hoeffding lower bound $B_H$ to perform better than the Markov lower bound $B_M$ when $\ell$ is larger, so we also consider the setting where $\ell=1000$.
In this setting, we took $k=25$ for NORM-25 and $k=10$ for the other four datasets.
For all of these experiments, we use sketch size $s=300$ and error rate $\epsilon = 0.01$.

\subsection{Results}
\label{sec.experiment results}
Before running our algorithms, we run the $k$-means++ algorithm $\ell$ times on the full dataset and record the smallest $k$-means value, denoted by $\min v_i$.
(This number of trials equals the number of random sketch trials for simplicity.)
We take this $k$-means value to be truncation level $u$ in our SDP truncation technique for $B_H$.
Note that $\min v_i$ is an upper bound on the optimal $k$-means objective.

For comparison, we use \eqref{eq.kmean++ lower bound} to get high-confidence lower bounds similar to $B_H$ (with value truncation) and $B_M$: 
\[
L_H:=\frac{1}{\ell}\sum_{i\in[\ell]} \min \{L_i,u\}-\Big(\tfrac{u^2}{2\ell}\log(\tfrac{1}{\epsilon})\Big)^{1/2},
\qquad
L_M:=\epsilon^{1/\ell}\min_{i\in[\ell]}L_i,
\]
where $\ell$, $u$, and $\epsilon$ are the same as for $B_H$ and $B_M$, and each $L_i$ is an independent draw of $L$.
By the proofs of Theorems~\ref{thm.hoefding monte carlo}(a) and~\ref{thm.markov monte carlo}(a), the random variables $L_H$ and $L_M$ are indeed lower bounds on $\operatorname{IP}(X,k)$ with probability $1-\epsilon$.
In Tables~\ref{tab:l30} and~\ref{tab:1000}, we include the average of $\{L_i\}_{i\in[\ell]}$ for comparison (denoted by $\operatorname{avg}L_i$), which can be seen as a lower bound without a probability guarantee.

Our results with $\ell=30$ are in Table~\ref{tab:l30}, while our results with $\ell=1000$ are in Table~\ref{tab:1000}.
The bold numbers give the best lower bound among $L_H$, $L_M$, $B_H$, and $B_M$.
In all datasets except INTRUSION, our lower bound (the best between $B_H$ and $B_M$) is at least $10$ times better than $L_H$ and $L_M$ (and is still a significant improvement over $\operatorname{avg}L_i$).
However, our lower bound is worse than the $k$-means++ lower bound for the INTRUSION dataset.
This is because INTRUSION is quite unbalanced: most of the data is concentrated at the same point, and there are a few very distant outliers.
Indeed, when we uniformly sketch down to a small subset, most of the outliers may not be selected, and so even the optimal $k$-means value of the subset $\operatorname{IP}(Y,k)$ is much smaller than $\operatorname{IP}(X,k)$.
(This can be demonstrated by running $k$-means++ on $Y$ for INTRUSION.)
As such, the bad performance on INTRUSION data is due to sketching rather than relaxing.
As expected, $B_M$ is slightly better when $\ell$ is small, while $B_H$ slightly better when $\ell$ is large.
In addition, we observe that $L_H$ is consistently negative (as is $B_H$ for INTRUSION).

Tables~\ref{tab:l30} and~\ref{tab:1000} also display three types of runtime.
$T_{\mathrm{init}}$ is the time for $\ell$ repeated initializations of $k$-means++, which is used to get $\{L_i\}_{i\in[\ell]}$.
$T_{\mathrm{k++}}$ is the time to compute $\min v_i$ (i.e., the threshold $u$).
$T_{\mathrm{SDP}}$ is the time to compute $\ell$ randomly sketched SDPs.
Based on the definitions of these lower bounds, the runtime for $L_H$ is $T_{\mathrm{init}} + T_{\mathrm{k++}}$, the runtime for $L_M$ is  $T_{\mathrm{init}}$, the runtime for $B_H$ is $T_{\mathrm{SDP}} + T_{\mathrm{k++}}$, and the runtime for $B_M$ is $T_{\mathrm{SDP}}$.
While the SDP approach takes longer than $k$-means++ in these examples, the SDP approach usually delivers a much better lower bound.


\begin{table}[h]
\centering
\caption{High-confidence $k$-means lower bounds with $\ell=30$.\label{tab:l30}}
\medskip
\footnotesize
\begin{tabular}{|l|lll|llll|lll|}\hline
Dataset & $k$ & $\min v_i$ & $\operatorname{avg}L_i$ & \phantom{-}$L_H$ & $L_M$ & \phantom{-}$B_H$ & $B_M$ & $T_{\mathrm{init}}$ & $T_{\mathrm{k++}}$ & $T_{\mathrm{SDP}}$ \\ \hline\hline
MNIST & 10 & 3.92e1 & 1.26e0 & -9.59e0 & 1.06e0 & \phantom{-}2.56e1 & \textbf{2.96e1} & 1.71e1 & 4.18e2 & 2.47e2\\\hline
NORM-10 & 10 & 4.97e0 & 1.10e0 & -1.07e0 & 1.24e-1 & \phantom{-}3.42e0 & \textbf{3.72e0} & 2.24e-1 & 1.48e-1 & 2.92e1\\
NORM-10 & 25 & 4.05e0 & 1.02e-1 & -1.01e0 & 8.63e-2 & \phantom{-}1.43e0 & \textbf{2.11e0} & 4.00e-1 & 1.97e0 & 3.98e2\\
NORM-10 & 50 & 3.18e0 & 7.58e-2 & -8.05e-1 & 6.33e-2 & \phantom{-}5.24e-1 & \textbf{1.12e0} & 6.30e-1 & 3.71e0 & 1.93e2\\\hline
NORM-25 & 10 & 1.18e5 & 3.94e3 & -2.90e4 & 3.03e3 & \phantom{-}6.99e4 & \textbf{8.08e4} & 2.58e-1 & 1.95e-1 & 1.67e2\\
NORM-25 & 25 & 1.50e1 & 6.37e0 & -3.31e0 & 3.08e-1 & \phantom{-}9.61e0 & \textbf{1.15e1} & 4.42e-1 & 3.63e-1 & 3.44e1\\
NORM-25 & 50 & 1.41e1 & 3.04e-1 & -3.61e0 & 2.60e-1 & \phantom{-}5.23e0 & \textbf{7.48e0} & 6.88e-1 & 2.50e0 & 4.64e1\\\hline
CLOUD & 10 & 5.62e3 & 2.41e2 & -1.31e3 & 1.62e2 & \phantom{-}2.70e3 & \textbf{3.06e3} & 1.01e-1 & 1.75e-1 & 1.46e2\\
CLOUD & 25 & 1.94e3 & 6.31e1 & -4.75e2 & 4.91e1 & \phantom{-}8.24e2 & \textbf{9.43e2} & 1.54e-1 & 1.90e-1 & 3.48e2\\
CLOUD & 50 & 1.09e3 & 2.99e1 & -2.72e2 & 2.33e1 & \phantom{-}2.57e2 & \textbf{4.54e2} & 2.19e-1 & 3.67e-1 & 4.99e2\\\hline
INTRUSION & 10 & 2.36e7 & 1.00e6 & -5.55e6 & \textbf{6.03e5} & -6.43e6 & 2.93e4 & 9.46e0 & 3.76e1 & 3.43e2\\
INTRUSION & 25 & 2.19e6 & 7.64e4 & -5.31e5 & \textbf{5.31e4} & -6.03e5 & 1.67e3 & 1.86e1& 1.04e2 & 4.58e2\\
INTRUSION & 50 & 4.52e5 & 1.36e4 & -1.11e5 & \textbf{9.27e3} & -1.25e5 & 4.56e1 & 3.43e1 & 1.51e2 & 2.16e3 \\\hline
\end{tabular}
\normalsize
\end{table}

\begin{table}[h]
\centering
\caption{High-confidence $k$-means lower bounds with $\ell=1000$.\label{tab:1000}}
\medskip
\footnotesize
\begin{tabular}{|l|lll|llll|lll|}\hline
Dataset & $k$ & $\min v_i$ & $\operatorname{avg}L_i$ & \phantom{-}$L_H$ & $L_M$ & \phantom{-}$B_H$ & $B_M$ & $T_{\mathrm{init}}$ & $T_{\mathrm{k++}}$ & $T_{\mathrm{SDP}}$ \\ \hline\hline
MNIST & 10 & 3.92e1 & 1.26e0 & -6.11e-1 & 1.21e0 & \textbf{\phantom{-}3.44e1} & 3.39e1 & 5.79e2 & 1.45e4 & 8.08e3\\
NORM-10 & 10 & 5.02e0 & 4.05e-1 & -8.05e-2 & 1.45e-1 & \textbf{\phantom{-}4.59e0} & 4.18e0 & 8.81e0 & 5.02e0 & 1.14e3\\
NORM-25 & 25 & 1.49e1 & 1.73e0 & -1.84e-1 & 3.56e-1 & \textbf{\phantom{-}1.29e1} & 1.27e1 & 1.56e1 & 1.14e1 & 1.26e3\\
CLOUD & 10 & 5.62e3 & 2.31e2 & -3.89e1 & 1.72e2 & \textbf{\phantom{-}4.06e3} & 3.14e3 & 4.05e0 & 6.10e0 & 6.27e3\\
INTRUSION & 10 & 2.34e7 & 9.58e5 & -1.66e5 & \textbf{6.80e5} & -1.00e6 & 1.81e4 & 3.15e2 & 2.44e3 & 2.65e4\\\hline
\end{tabular}
\normalsize
\end{table}

\section{Discussion}
\label{sec.discussion}
In this paper, we introduced sketch-and-solve approaches to $k$-means clustering with semidefinite programming.
In particular, we exactly recover the optimal clustering with a \textit{single sketch} provided the data is well separated, and we compute a high-confidence lower bound on the $k$-means value from \textit{multiple sketches} when the data is not well separated.
For future work, one might attempt to use multiple sketches to find the optimal clustering when the data only exhibits \textit{some} separation.
Next, our high-confidence lower bounds perform poorly for unbalanced data like INTRUSION, and we suspect this stems from our uniform sampling approach.
Presumably, an \textit{importance sampling}--based alternative would perform better in such settings.
Finally, the main idea of our high-confidence lower bounds is that $\mathbb{E}\operatorname{IP}(Y,k)
\leq\operatorname{IP}(X,k)$ (see Lemma~\ref{lem.ip sketch bound}).
It would be interesting to show a similar relationship for SDP, namely, $\mathbb{E}\operatorname{SDP}(Y,k)
\leq\operatorname{SDP}(X,k)$.
While this bound holds empirically, we do not have a proof.
Such a bound might allow one to extend our sketch-and-solve approach to more general SDPs.

\section{Proof of Theorem~\ref{thm.sketch-and-solve-exact}}
\label{sec.proof thm2}

For each $S\in\Gamma$, denote $S':=S\cap W$.
To prove Theorem~\ref{thm.sketch-and-solve-exact}, we will first find sufficient conditions for the following approximations to hold:
\[
\alpha_{S'T'}\gtrapprox\alpha_{ST},
\qquad
\beta_{S'T'}\approx\beta_{ST},
\qquad
c_{S'}\approx c_S.
\]
The first two approximations combined with Proposition~\ref{prop.prox} ensure that the SDP step of Algorithm~\ref{alg.exact} produces the clustering $\Gamma'=\{S':S\in\Gamma\}$.
Meanwhile, the approximation $c_{S'}\approx c_S$ ensures that the final step of Algorithm~\ref{alg.exact} produces the desired clustering $\Gamma$. 

\begin{lemma}
\label{lem.scalar bernstein}
For $X\sim\mathsf{Binomial}(n,p)$, it holds that
\[
\mathbb{P}\{X\leq\tfrac{1}{2}pn\}\leq\operatorname{exp}(-\tfrac{3}{28}pn),
\qquad
\mathbb{P}\{X\geq\tfrac{3}{2}pn\}\leq\operatorname{exp}(-\tfrac{3}{28}pn).
\]
\end{lemma}

\begin{proof}
This is an immediate consequence of Bernstein's inequality:
\[
\mathbb{P}\{X\leq\tfrac{1}{2}pn\}
=\mathbb{P}\{pn-X\geq\tfrac{1}{2}pn\}
\leq\operatorname{exp}\bigg(-\frac{\frac{1}{2}(\frac{1}{2}pn)^2}{np(1-p)+\frac{1}{3}(\frac{1}{2}pn)}\bigg)
\leq\operatorname{exp}(-\tfrac{3}{28}pn).
\]
The other bound follows from a similar proof.
\end{proof}

\begin{proposition}[Matrix Bernstein, see Theorem~6.6.1 in~\cite{Tropp:15}]
\label{prop.matrix bernstein}
Consider independent, random, mean-zero, real symmetric $d\times d$ matrices $\{X_i\}_{i\in[n]}$ such that
\[
\lambda_{\mathrm{max}}(X_i)\leq L ~~~\forall i\in[n],
\qquad
\qquad
\bigg\|\sum_{i\in[n]}\mathbb{E}X_i^2\bigg\|_{2\to2}\leq v.
\]
Then for every $t\geq0$, it holds that
\[
\mathbb{P}\bigg\{\lambda_{\mathrm{max}}\bigg(\sum_{i\in[n]}X_i\bigg)\geq t\bigg\}
\leq d\cdot\operatorname{exp}\Big(-\tfrac{\frac{1}{2}t^2}{v+\frac{1}{3}Lt}\Big)
\leq d\cdot\operatorname{exp}(-\tfrac{1}{4}\min\{\tfrac{t^2}{v},\tfrac{3t}{L}\}).
\]
\end{proposition}

\begin{lemma}
\label{lem.vector bernstein}
Given independent random vectors $\{X_i\}_{i\in[n]}$ satisfying $\mathbb{E}X_i=0$ and $\|X_i\|\leq r$ almost surely for each $i\in[n]$, put $v:=\max_{i\in[n]}\mathbb{E}\|X_i\|^2$.
Then for $t\geq0$, it holds that
\[
\mathbb{P}\bigg\{\bigg\|\sum_{i\in[n]}X_i\bigg\|\geq t\bigg\}
\leq(d+1)\operatorname{exp}\Big(-\tfrac{\frac{1}{2}t^2}{nv+\frac{1}{3}rt}\Big)
\leq (d+1)\operatorname{exp}(-\tfrac{1}{4}\min\{\tfrac{t^2}{nv},\tfrac{3t}{r}\}).
\]
\end{lemma}

\begin{proof}
Apply matrix Bernstein to the random matrices $\left[\begin{smallmatrix}0&X_i^\top\\X_i&0\end{smallmatrix}\right]$.
\end{proof}

\begin{lemma}
\label{lem.centroid}
Given a tuple $\{x_i\}_{i\in[n]}$ of points in $\mathbb{R}^d$, consider their centroid and radius:
\[
c:=\frac{1}{n}\sum_{i\in[n]}x_i,
\qquad
r:=\max_{i\in[n]}\|x_i-c\|.
\]
Let $\{b_i\}_{i\in[n]}$ denote independent Bernoulli random variables with success probability $p$, and consider the random variable $n':=\sum_{i\in[n]}b_i$ and the random vector
\[
c':=\left\{\begin{array}{cl}
\frac{1}{n'}\sum_{i\in[n]}b_ix_i&\text{if }n'>0\\
\text{undefined}&\text{if }n'=0.
\end{array}\right.
\]
Provided $pn\geq 16\log(\frac{d+2}{\epsilon_{\ref{lem.centroid}}})$, it holds that $\|c'-c\|^2<\frac{16r^2}{pn}\log(\frac{d+2}{\epsilon_{\ref{lem.centroid}}})$ with probability $\geq1-\epsilon_{\ref{lem.centroid}}$.
\end{lemma}

\begin{proof}
In the event $\{n'>0\}$, it holds that
\[
\|c'-c\|
=\bigg\|\frac{1}{n'}\sum_{i\in[n]}b_ix_i-c\bigg\|
=\frac{1}{n'}\bigg\|\sum_{i\in[n]}b_i(x_i-c)\bigg\|
=\frac{1}{n'}\bigg\|\sum_{i\in[n]}(b_i-p)(x_i-c)\bigg\|,
\]
where the last step applies the fact that $\sum_{i\in[n]}(x_i-c)=0$.
Then
\begin{align}
\mathbb{P}\{\|c'-c\|\geq t\}
\nonumber
&\leq \mathbb{P}(\{\|c'-c\|\geq t\}\cap\{n'>\tfrac{1}{2}pn\})+\mathbb{P}\{n'\leq\tfrac{1}{2}pn\}\\
\label{eq.centroid1}
&\leq \mathbb{P}\bigg\{\frac{2}{pn}\bigg\|\sum_{i\in[n]}(b_i-p)(x_i-c)\bigg\|\geq t\bigg\}+\mathbb{P}\{n'\leq\tfrac{1}{2}pn\}.
\end{align}
We apply Lemma~\ref{lem.vector bernstein} to the first term above and Lemma~\ref{lem.scalar bernstein} to the second term.
Specifically, put $X_i:=(b_i-p)(x_i-c)$.
Then $\mathbb{E}X_i=0$ and $\|X_i\|\leq r$ almost surely.
Furthermore,
\[
\mathbb{E}\|X_i\|^2
=\|x_i-c\|^2\cdot\mathbb{E}(b_i-p)^2
\leq pr^2,
\]
and so $v\leq pr^2$.
With this, we continue to bound \eqref{eq.centroid1}:
\begin{align}
\mathbb{P}\{\|c'-c\|\geq t\}
\nonumber
&\leq (d+1)\operatorname{exp}\Big(-\tfrac{1}{4}\min\Big\{\tfrac{(\frac{pnt}{2})^2}{npr^2},\tfrac{3(\frac{pnt}{2})}{r}\Big\}\Big)+\operatorname{exp}(-\tfrac{3}{28}pn)\\
\label{eq.important bound for later}
&\leq (d+2)\operatorname{exp}(-\tfrac{pn}{4}\min\{\tfrac{t^2}{4r^2},\tfrac{3t}{2r},\tfrac{3}{7}\})
= (d+2)\operatorname{exp}(-\tfrac{pnt^2}{16r^2}),
\end{align}
where the last step holds provided $t\leq r$.
The result follows by taking $t^2:=\frac{16r^2}{pn}\log(\frac{d+2}{\epsilon_{\ref{lem.centroid}}})$, since then our assumption $pn\geq 16\log(\frac{d+2}{\epsilon_{\ref{lem.centroid}}})$ implies $t\leq r$.
\end{proof}

Here and throughout, we denote
\[
n_{\mathrm{min}}:=\min_{S\in\Gamma}|S|,
\qquad
n_{\mathrm{max}}:=\max_{S\in\Gamma}|S|.
\]

\begin{lemma}
\label{lem.lower bound alpha}
Fix $S,T\in\Gamma$ with $S\neq T$ and suppose $pn_{\mathrm{min}}\geq16\log(\frac{2(d+2)}{\epsilon_{\ref{lem.lower bound alpha}}})$.
Then
\[
\alpha_{S'T'}>
\alpha_{ST}-(\tfrac{2r}{\Delta}+\tfrac{3}{2})\cdot 8r \cdot\sqrt{\tfrac{\log(2(d+2)/\epsilon_{\ref{lem.lower bound alpha}})}{pn_{\mathrm{min}}}}
\]
with probability $\geq1-\epsilon_{\ref{lem.lower bound alpha}}$.
\end{lemma}

\begin{proof}
Denote $m_{ST}:=\tfrac{c_S+c_T}{2}$ and $w_{ST}:=\tfrac{c_S-c_T}{\|c_S-c_T\|}$ so that $\alpha_{ST}=\min_{i\in S}\langle x_i-m_{ST},w_{ST}\rangle$, and let $j$ denote any minimizer of $\langle x_i-m_{S'T'},w_{S'T'}\rangle$ over $i\in S'$.
Then
\begin{align}
\alpha_{ST}-\alpha_{S'T'}
\nonumber
&\leq\langle x_j-m_{ST},w_{ST}\rangle-\langle x_j-m_{S'T'},w_{S'T'}\rangle\\
\nonumber
&=\langle x_j-m_{ST},w_{ST}-w_{S'T'}\rangle+\langle m_{S'T'}-m_{ST},w_{S'T'}\rangle\\
\label{eq.lhs1}
&\leq\|x_j-m_{ST}\|\|w_{S'T'}-w_{ST}\|+\|m_{S'T'}-m_{ST}\|.
\end{align}
To continue, we bound each of the terms above.
First, the triangle inequality gives
\[
\|x_j-m_{ST}\|
=\|x_j-c_S+c_S-m_{ST}\|
\leq \|x_j-c_S\|+\|c_S-m_{ST}\|
\leq r+\tfrac{1}{2}\|c_S-c_T\|,
\]
\[
\|m_{S'T'}-m_{ST}\|
=\tfrac{1}{2}\|(c_{S'}+c_{T'})-(c_{S}+c_{T})\|
\leq\tfrac{1}{2}\big(\|c_{S'}-c_S\|+\|c_{T'}-c_T\|\big).
\]
Next, we apply the triangle inequality multiple times to get
\begin{align*}
\|w_{S'T'}-w_{ST}\|
&=\|\tfrac{c_{S'}-c_{T'}}{\|c_{S'}-c_{T'}\|}-\tfrac{c_{S'}-c_{T'}}{\|c_{S}-c_{T}\|}+\tfrac{c_{S'}-c_{T'}}{\|c_{S}-c_{T}\|}-\tfrac{c_{S}-c_{T}}{\|c_{S}-c_{T}\|}\|\\
&\leq\|\tfrac{c_{S'}-c_{T'}}{\|c_{S'}-c_{T'}\|}-\tfrac{c_{S'}-c_{T'}}{\|c_{S}-c_{T}\|}\|+\|\tfrac{c_{S'}-c_{T'}}{\|c_{S}-c_{T}\|}-\tfrac{c_{S}-c_{T}}{\|c_{S}-c_{T}\|}\|\\
&=\tfrac{|\|c_{S}-c_{T}\|-\|c_{S'}-c_{T'}\||}{\|c_{S}-c_{T}\|}+\tfrac{\|c_{S'}-c_{T'}-c_{S}+c_{T}\|}{\|c_{S}-c_{T}\|}
\leq2\cdot\tfrac{\|c_{S'}-c_S\|+\|c_{T'}-c_T\|}{\|c_{S}-c_{T}\|}.
\end{align*}
With this, we continue \eqref{eq.lhs1}:
\begin{align*}
\alpha_{ST}-\alpha_{S'T'}
&\leq\big(r+\tfrac{1}{2}\|c_S-c_T\|\big)\cdot2\cdot\tfrac{\|c_{S'}-c_S\|+\|c_{T'}-c_T\|}{\|c_{S}-c_{T}\|}+\tfrac{1}{2}\big(\|c_{S'}-c_S\|+\|c_{T'}-c_T\|\big)\\
&=\big(\tfrac{2r}{\|c_S-c_T\|}+\tfrac{3}{2}\big)\big(\|c_{S'}-c_S\|+\|c_{T'}-c_T\|\big)\\
&\leq\big(\tfrac{2r}{\Delta}+\tfrac{3}{2}\big)\big(\|c_{S'}-c_S\|+\|c_{T'}-c_T\|\big).
\end{align*}
It follows that
\begin{align*}
\mathbb{P}\{\alpha_{ST}-\alpha_{S'T'}\geq t\}
&\leq\mathbb{P}\{\big(\tfrac{2r}{\Delta}+\tfrac{3}{2}\big)\big(\|c_{S'}-c_S\|+\|c_{T'}-c_T\|\big)\geq t\}\\
&\leq\mathbb{P}\{\big(\tfrac{2r}{\Delta}+\tfrac{3}{2}\big)\|c_{S'}-c_S\|\geq \tfrac{t}{2}\}+\mathbb{P}\{\big(\tfrac{2r}{\Delta}+\tfrac{3}{2}\big)\|c_{T'}-c_T\|\geq \tfrac{t}{2}\}\\
&\leq2\max_{\substack{R\in\Gamma}}\mathbb{P}\{\big(\tfrac{2r}{\Delta}+\tfrac{3}{2}\big)\|c_{R'}-c_R\|\geq \tfrac{t}{2}\}.
\end{align*}
The result then follows from Lemma~\ref{lem.centroid} by taking $\epsilon_{\ref{lem.centroid}}:=\epsilon_{\ref{lem.lower bound alpha}}/2$.
\end{proof}

\begin{lemma}
\label{lem.deviation in sum of reciprocals}
Fix $S,T\in\Gamma$ with $S\neq T$ and suppose $pn_{\mathrm{min}}\geq\frac{104}{3}(\frac{n_{\mathrm{max}}}{n_{\mathrm{min}}})^2\log(\frac{6}{\epsilon_{\ref{lem.deviation in sum of reciprocals}}})$.
Then
\[
|p(\tfrac{1}{|S'|}+\tfrac{1}{|T'|})-(\tfrac{1}{|S|}+\tfrac{1}{|T|})|
<\sqrt{\tfrac{104}{3}\tfrac{\log(6/\epsilon_{\ref{lem.deviation in sum of reciprocals}})}{pn_{\mathrm{min}}^3}}
\]
with probability $\geq1-\epsilon_{\ref{lem.deviation in sum of reciprocals}}$.
\end{lemma}

\begin{proof}
First, the triangle inequality gives
\[
|p(\tfrac{1}{|S'|}+\tfrac{1}{|T'|})-(\tfrac{1}{|S|}+\tfrac{1}{|T|})|
\leq|\tfrac{p}{|S'|}-\tfrac{1}{|S|}|+|\tfrac{p}{|T'|}-\tfrac{1}{|T|}|.
\]
As such, it suffices to bound terms of the form
\[
|\tfrac{p}{|S'|}-\tfrac{1}{|S|}|
=\tfrac{|p|S|-|S'||}{|S'||S|}
\leq\tfrac{|p|S|-|S'||}{\frac{p}{2}|S|^2},
\]
where the last step holds in the event $\{|S'|\geq\frac{p}{2}|S|\}$.
For every $t\in[0,\frac{1}{|S|}]$, Bernstein's inequality and Lemma~\ref{lem.scalar bernstein} together give
\begin{align*}
\mathbb{P}\big\{|\tfrac{p}{|S'|}-\tfrac{1}{|S|}|\geq\tfrac{t}{2}\big\}
&\leq\mathbb{P}\big\{\big||S'|-p|S|\big|\geq\tfrac{p}{4}|S|^2t\big\}+\mathbb{P}\big\{|S'|\leq\tfrac{p}{2}|S|\big\}\\
&\leq 2\operatorname{exp}\Big(-\tfrac{\frac{1}{2}(\frac{p}{4}|S|^2t)^2}{|S|p(1-p)+\frac{1}{3}(\frac{p}{4}|S|^2t)}\Big)+\operatorname{exp}(-\tfrac{3}{28}p|S|)\\
&\leq 2\operatorname{exp}\Big(-\tfrac{\frac{1}{2}(\frac{p}{4}|S|^2t)^2}{\frac{13}{12}p|S|}\Big)+\operatorname{exp}(-\tfrac{3}{28}p|S|)\\
&= 2\operatorname{exp}(-\tfrac{3}{104}p|S|^3t^2)+\operatorname{exp}(-\tfrac{3}{28}p|S|)
\leq 3\operatorname{exp}(-\tfrac{3}{104}p|S|^3t^2).
\end{align*}
Finally, we combine our estimates to obtain
\begin{align*}
\mathbb{P}\big\{|p(\tfrac{1}{|S'|}+\tfrac{1}{|T'|})-(\tfrac{1}{|S|}+\tfrac{1}{|T|})|\geq t\big\}
&\leq\mathbb{P}\big\{|\tfrac{p}{|S'|}-\tfrac{1}{|S|}|\geq\tfrac{t}{2}\big\}+\mathbb{P}\big\{|\tfrac{p}{|T'|}-\tfrac{1}{|T|}|\geq\tfrac{t}{2}\big\}\\
&\leq 3\operatorname{exp}(-\tfrac{3}{104}p|S|^3t^2)+3\operatorname{exp}(-\tfrac{3}{104}p|T|^3t^2)\\
&\leq 6\operatorname{exp}(-\tfrac{3}{104}pn_\mathrm{min}^3t^2).
\end{align*}
The result follows by taking $t:=\sqrt{\tfrac{104}{3}\tfrac{\log(6/\epsilon_{\ref{lem.deviation in sum of reciprocals}})}{pn_{\mathrm{min}}^3}}$, which is at most $\frac{1}{n_{\mathrm{max}}}$ by assumption.
\end{proof}

\begin{lemma}
\label{lem.deviation in spectral norm}
Fix $R\in\Gamma$.
Then for every $t\in[0,p|R|r^2]$, it holds that
\[
\mathbb{P}\Big\{\big|\|X_{R'}\|_{2\to2}^2-p\|X_R\|_{2\to2}^2\big|\geq t\Big\}
\leq(3d+3)\cdot\operatorname{exp}(-\tfrac{t^2}{48p|R|r^4}).
\]
\end{lemma}

\begin{proof}
Fix $R\in\Gamma$.
For each $i\in R$, let $b_i\sim\mathsf{Bernoulli}(p)$ indicate whether $i\in R'$, and consider the matrices
\[
X_R:=\sum_{i\in R}(x_i-c_R)e_i^\top,
\qquad
X_{R'}:=\sum_{i\in R}b_i(x_i-c_{R'})e_i^\top,
\qquad
W_R:=\sum_{i\in R}b_i(x_i-c_{R})e_i^\top.
\]
(Here and throughout, we note that $c_{R'}$ and any quantity defined in terms of $c_{R'}$ is undefined in the event that $R'$ is empty.)
We will bound the deviation of $\|X_{R'}\|_{2\to2}^2$ from $p\|X_R\|_{2\to2}^2$ by applying the triangle inequality through $\|W_R\|_{2\to2}^2$.
To facilitate this analysis, put $e_R:=c_{R'}-c_R$.
Then $X_{R'}=W_R-e_R1_{R'}^\top$ and $W_R1_{R'}=|R'|e_R$, from which it follows that
\[
X_{R'}X_{R'}^\top-W_RW_R^\top
=-|R'|e_Re_R^\top.
\]
As such, the triangle and reverse triangle inequalities together give
\begin{align*}
\big|\|X_{R'}\|_{2\to2}^2-p\|X_R\|_{2\to2}^2\big|
&\leq\big|\|X_{R'}\|_{2\to2}^2-\|W_R\|_{2\to2}^2\big|+\big|\|W_R\|_{2\to2}^2-p\|X_R\|_{2\to2}^2\big|\\
&=\big|\|X_{R'}X_{R'}^\top\|_{2\to2}-\|W_RW_R^\top\|_{2\to2}\big|+\big|\|W_RW_R^\top\|_{2\to2}-\|pX_RX_R^\top\|_{2\to2}\big|\\
&\leq\|X_{R'}X_{R'}^\top-W_RW_R^\top\|_{2\to2}+\|W_RW_R^\top-pX_RX_R^\top\|_{2\to2}\\
&=|R'|\|e_R\|^2+\|W_RW_R^\top-pX_RX_R^\top\|_{2\to2}.
\end{align*}
We will use Lemma~\ref{lem.scalar bernstein} to bound $|R'|$, Lemma~\ref{lem.centroid} to bound $\|e_R\|^2$, and Proposition~\ref{prop.matrix bernstein} to bound $\|W_RW_R^\top-pX_RX_R^\top\|_{2\to2}$.
For this third bound, we put $X_i:=(b_i-p)(x_i-c_R)(x_i-c_R)^\top$ and observe that
\[
\|X_i\|_{2\to2}
\leq\|x_i-c_R\|^2
\leq r^2
=:L,
\]
\[
\bigg\|\sum_{i\in R}\mathbb{E}X_i^2\bigg\|_{2\to2}
\leq\operatorname{tr}\bigg(\sum_{i\in R}\mathbb{E}X_i^2\bigg)
=p(1-p)\sum_{i\in R}\|x_i-c_R\|^4
\leq p|R|r^4
=:v.
\]
As such, Lemma~\ref{lem.scalar bernstein}, the bound~\eqref{eq.important bound for later} (which similarly holds with our current choice of $r$), and Proposition~\ref{prop.matrix bernstein} together give
\begin{align*}
&\mathbb{P}\Big\{\big|\|X_{R'}\|_{2\to2}^2-p\|X_R\|_{2\to2}^2\big|\geq t\Big\}\\
&\qquad\leq\mathbb{P}\Big\{|R'|\geq\tfrac{3}{2}p|R|\Big\}+\mathbb{P}\Big\{\tfrac{3}{2}p|R|\cdot\|e_R\|^2\geq\tfrac{t}{2}\Big\}+\mathbb{P}\bigg\{\Big\|\sum_{i\in R}X_i\Big\|_{2\to2}
\geq\tfrac{t}{2}\bigg\}\\
&\qquad\leq\operatorname{exp}(-\tfrac{3}{28}p|R|)+(d+2)\operatorname{exp}\Big(-\tfrac{p|R|(\frac{t}{3p|R|})}{16r^2}\Big)+2d\cdot\operatorname{exp}\Big(-\tfrac{1}{4}\min\Big\{\tfrac{(\frac{t}{2})^2}{|R|pr^4},\tfrac{3(\frac{t}{2})}{r^2}\Big\}\Big)\\
&\qquad\leq(3d+3)\cdot\operatorname{exp}\Big(-\min\Big\{\tfrac{3p|R|}{28},\tfrac{t}{48r^2},\tfrac{t^2}{16p|R|r^4},\tfrac{3t}{8r^2}\Big\}\Big)
\leq(3d+3)\cdot\operatorname{exp}(-\tfrac{t^2}{48p|R|r^4}),
\end{align*}
where the last step holds provided $t\in[0,p|R|r^2]$.
\end{proof}

\begin{lemma}
\label{lem.deviation in beta}
Fix $S,T\in\Gamma$ with $S\neq T$ and suppose $pn_{\mathrm{min}}\geq 48(\frac{n_{\mathrm{max}}}{n_{\mathrm{min}}})^2\log(\frac{9k(d+1)}{\epsilon_{\ref{lem.deviation in beta}}})$.
Then
\[
|\beta_{S'T'}-\beta_{ST}|
\leq \tfrac{1}{2}\bigg(\Big(16\sqrt{3}+\sqrt{\tfrac{104}{3}}\Big)\cdot\tfrac{k}{\sqrt{pn_{\mathrm{min}}}}\cdot\tfrac{n_{\mathrm{max}}}{n_{\mathrm{min}}}\cdot r^2\cdot\log^{1/2}(\tfrac{9k(d+1)}{\epsilon_{\ref{lem.deviation in beta}}})\bigg)^{1/2}
\]
with probability $\geq1-\epsilon_{\ref{lem.deviation in beta}}$.
\end{lemma}

\begin{proof}
Put $U:=p(\tfrac{1}{|S'|}+\tfrac{1}{|T'|})$, $V:=\frac{1}{p}\sum_{R\in\Gamma}\|X_{R'}\|_{2\to2}^2$, $u:=\frac{1}{|S|}+\frac{1}{|T|}$, and $v:=\sum_{R\in\Gamma}\|X_{R}\|_{2\to2}^2$.
Then
\begin{align*}
|\beta_{S'T'}-\beta_{ST}|
&=|\tfrac{1}{2}(UV)^{1/2}-\tfrac{1}{2}(uv)^{1/2}|\\
&\leq\tfrac{1}{2}|UV-uv|^{1/2}
=\tfrac{1}{2}|U(V-v)+v(U-u)|^{1/2}
\leq\tfrac{1}{2}(U|V-v|+v|U-u|)^{1/2},
\end{align*}
where the first inequality follows from the fact that $x,y\geq0$ implies
\[
|x-y|^2
=x^2-2xy+y^2
\leq x^2-2\min\{x^2,y^2\}+y^2
=|x^2-y^2|,
\]
while the second inequality follows from the triangle inequality.
Thus, it suffices to bound $U$, $|V-v|$, $v$, and $|U-u|$ in a high-probability event.
First, Lemma~\ref{lem.scalar bernstein} gives that $|S'|>\frac{1}{2}p|S|$ with probability $\geq1-\exp(-\tfrac{3}{28}p|S|)$, and similarly, $|T'|>\frac{1}{2}p|T|$ with probability $\geq1-\exp(-\tfrac{3}{28}p|T|)$.
A union bound therefore gives
\[
U
=p(\tfrac{1}{|S'|}+\tfrac{1}{|T'|})
<p(\tfrac{2}{p|S|}+\tfrac{2}{p|T|})
\leq\tfrac{4}{n_{\mathrm{min}}}
\]
with probability $\geq1-2\exp(-\tfrac{3}{28}pn_{\mathrm{min}})$.
Next, we apply the triangle inequality, union bound, and Lemma~\ref{lem.deviation in spectral norm} to get
\[
|V-v|
=\bigg|\frac{1}{p}\sum_{R\in\Gamma}\|X_{R'}\|_{2\to2}^2-\sum_{R\in\Gamma}\|X_R\|_{2\to2}^2\bigg|
\leq \frac{1}{p}\sum_{R\in\Gamma}\Big|\|X_{R'}\|_{2\to2}^2-p\|X_R\|_{2\to2}^2\Big|
<\frac{kt}{p}
\]
with probability $\geq1-3k(d+1)\exp(-\frac{t^2}{48pn_\mathrm{max}r^4})$, provided $t\in[0,pn_{\mathrm{min}}r^2]$.
For $v$, we pass to the Frobenius norm:
\[
v
=\sum_{R\in\Gamma}\|X_R\|_{2\to2}^2
\leq\sum_{R\in\Gamma}\|X_R\|_F^2
\leq nr^2
\leq kn_{\mathrm{max}}r^2.
\]
Finally, we apply Lemma~\ref{lem.deviation in sum of reciprocals} to obtain
\[
|U-u|
=|p(\tfrac{1}{|S'|}+\tfrac{1}{|T'|})-(\tfrac{1}{|S|}+\tfrac{1}{|T|})|
<\sqrt{\tfrac{104}{3}\tfrac{\log(6/\epsilon_{\ref{lem.deviation in sum of reciprocals}})}{pn_{\mathrm{min}}^3}}
\]
with probability $\geq1-\epsilon_{\ref{lem.deviation in sum of reciprocals}}$, provided $pn_{\mathrm{min}}\geq\frac{104}{3}(\frac{n_{\mathrm{max}}}{n_{\mathrm{min}}})^2\log(\frac{6}{\epsilon_{\ref{lem.deviation in sum of reciprocals}}})$.
We will combine these bounds using a union bound.
To do so, we will bound the failure probabilities corresponding to $U$, $|V-v|$, and $|U-u|$ by $\tfrac{\epsilon_{\ref{lem.deviation in beta}}}{3}$:
\begin{equation}
\label{eq.bounds on eps of lemma}
\tfrac{\epsilon_{\ref{lem.deviation in beta}}}{3}
\geq2\exp(-\tfrac{3}{28}pn_{\mathrm{min}}),
\qquad
\tfrac{\epsilon_{\ref{lem.deviation in beta}}}{3}
\geq3k(d+1)\exp(-\tfrac{t^2}{48pn_\mathrm{max}r^4}),
\qquad
\tfrac{\epsilon_{\ref{lem.deviation in beta}}}{3}
\geq\epsilon_{\ref{lem.deviation in sum of reciprocals}}.
\end{equation}
Rearranging the first bound in \eqref{eq.bounds on eps of lemma} gives $pn_{\mathrm{min}}\geq\tfrac{28}{3}\log(\frac{6}{\epsilon_{\ref{lem.deviation in beta}}})$, which is implied by our hypothesis.
We change the second bound in \eqref{eq.bounds on eps of lemma} to an equality that defines $t$ as
\[
t:=\sqrt{48pn_{\mathrm{max}}r^4\log(\tfrac{9k(d+1)}{\epsilon_{\ref{lem.deviation in beta}}})}.
\]
This choice satisfies the requirement that $t\in[0,pn_{\mathrm{min}}r^2]$ precisely when
\[
pn_{\mathrm{min}}\geq48\cdot\tfrac{n_{\mathrm{max}}}{n_{\mathrm{min}}}\cdot\log(\tfrac{9k(d+1)}{\epsilon_{\ref{lem.deviation in beta}}}),
\]
which is implied by our hypothesis.
Finally, we change the third bound in \eqref{eq.bounds on eps of lemma} to an equality that defines $\epsilon_{\ref{lem.deviation in sum of reciprocals}}:=\tfrac{\epsilon_{\ref{lem.deviation in beta}}}{3}$, which satisfies the requirement $pn_{\mathrm{min}}\geq\frac{104}{3}(\frac{n_{\mathrm{max}}}{n_{\mathrm{min}}})^2\log(\frac{6}{\epsilon_{\ref{lem.deviation in sum of reciprocals}}})$ by our hypothesis.
Putting everything together, we have
\begin{align*}
|\beta_{S'T'}-\beta_{ST}|
&\leq\tfrac{1}{2}(U|V-v|+v|U-u|)^{1/2}\\
&<\tfrac{1}{2}\bigg(\tfrac{4}{n_{\mathrm{min}}}\cdot\tfrac{k}{p}\sqrt{48pn_{\mathrm{max}}r^4\log(\tfrac{9k(d+1)}{\epsilon_{\ref{lem.deviation in beta}}})}+kn_{\mathrm{max}}r^2\cdot\sqrt{\tfrac{104}{3}\tfrac{\log(18/\epsilon_{\ref{lem.deviation in beta}})}{pn_{\mathrm{min}}^3}}\bigg)^{1/2}\\
&\leq\tfrac{1}{2}\bigg(\Big(16\sqrt{3}+\sqrt{\tfrac{104}{3}}\Big)\cdot\tfrac{k}{\sqrt{pn_{\mathrm{min}}}}\cdot\tfrac{n_{\mathrm{max}}}{n_{\mathrm{min}}}\cdot r^2\cdot\log^{1/2}(\tfrac{9k(d+1)}{\epsilon_{\ref{lem.deviation in beta}}})\bigg)^{1/2},
\end{align*}
where $16\sqrt{3}$ comes from the first term and $\sqrt{\tfrac{104}{3}}$ comes from the second term.
(To be clear, we used the fact that $n_{\mathrm{max}}\geq n_{\mathrm{min}}$ to bound the first term, and the fact that $9k(d+1)\geq18$ to bound the second term.)
\end{proof}

\begin{lemma}
\label{lem.approx prox condition}
Suppose $pn_{\mathrm{min}}\geq 48(\frac{n_{\mathrm{max}}}{n_{\mathrm{min}}})^2\log(\frac{18k^3(d+1)}{\epsilon_{\ref{lem.approx prox condition}}})$ and
\begin{align*}
\operatorname{prox}(X,\Gamma)
&>(\tfrac{2r}{\Delta}+\tfrac{3}{2})\cdot 8r \cdot\sqrt{\tfrac{\log(4k^2(d+2)/\epsilon_{\ref{lem.approx prox condition}})}{pn_{\mathrm{min}}}}\\
&\qquad+\tfrac{1}{2}\bigg(\Big(16\sqrt{3}+\sqrt{\tfrac{104}{3}}\Big)\cdot\tfrac{k}{\sqrt{pn_{\mathrm{min}}}}\cdot\tfrac{n_{\mathrm{max}}}{n_{\mathrm{min}}}\cdot r^2\cdot\log^{1/2}(\tfrac{18k^3(d+1)}{\epsilon_{\ref{lem.approx prox condition}}})\bigg)^{1/2}.
\end{align*}
Then $\operatorname{prox}(X',\Gamma')>0$ with probability $\geq1-\epsilon_{\ref{lem.approx prox condition}}$.
\end{lemma}

\begin{proof}
We apply Lemmas~\ref{lem.lower bound alpha} and~\ref{lem.deviation in beta} with $\epsilon_{\ref{lem.lower bound alpha}}=\epsilon_{\ref{lem.deviation in beta}}:=\frac{\epsilon_{\ref{lem.approx prox condition}}}{2k^2}$.
By taking union bound over $S,T\in\Gamma$ with $S\neq T$, the random variable $\operatorname{prox}(X,\Gamma)-\operatorname{prox}(X',\Gamma')$ is at most the right-hand side of the displayed inequality with probability $\geq1-\epsilon_{\ref{lem.approx prox condition}}$, and the result follows.
\end{proof}

\begin{lemma}
\label{lem.rounding after sketch and solve}
Suppose $r\leq\frac{\Delta}{2}$ and $pn_{\mathrm{min}}\geq16\max\{1,(\frac{r}{\Delta/2-r})^2\}\log(\tfrac{k(d+2)}{\epsilon_{\ref{lem.rounding after sketch and solve}}})$.
Then with probability $\geq1-\epsilon_{\ref{lem.rounding after sketch and solve}}$, it simultaneously holds that
\[
\|x_i-c_{S'}\|<\|x_i-c_{T'}\|
\]
for every $S,T\in\Gamma$ with $S\neq T$ and every $i\in S$.
\end{lemma}

\begin{proof}
We will show that a stronger condition holds, namely that
\begin{equation}
\label{eq.stronger condition}
\max_{\substack{S,T\in\Gamma\\S\neq T}}\max_{i\in S}\|x_i-c_{S'}\|
<\min_{\substack{S,T\in\Gamma\\S\neq T}}\min_{i\in S}\|x_i-c_{T'}\|
\end{equation}
with probability $\geq1-\epsilon_{\ref{lem.rounding after sketch and solve}}$.
Denote the random variable $E:=\max_{S\in\Gamma}\|c_{S'}-c_S\|$.
Then
\begin{align*}
\|x_i-c_{S'}\|
&\leq \|x_i-c_S\|+\|c_{S'}-c_S\|
\leq r+E,\\
\|x_i-c_{T'}\|
&\geq\|c_S-c_T\|-\|c_{T'}-c_T\|-\|x_i-c_S\|
\geq\Delta-E-r.
\end{align*}
Thus, the desired inequality \eqref{eq.stronger condition} holds whenever $E<\frac{\Delta}{2}-r$.
The result then follows from Lemma~\ref{lem.centroid} by taking $\epsilon_{\ref{lem.centroid}}:=\frac{\epsilon_{\ref{lem.rounding after sketch and solve}}}{k}$ and applying a union bound over $S\in\Gamma$.
\end{proof}

\begin{proof}[Proof of Theorem~\ref{thm.sketch-and-solve-exact}] Lemmas~\ref{lem.approx prox condition} and~\ref{lem.rounding after sketch and solve} with $\epsilon_{\ref{lem.approx prox condition}}=\epsilon_{\ref{lem.rounding after sketch and solve}}:=\frac{\epsilon}{2}$  together imply that Algorithm~\ref{alg.exact} exactly recovers $\Gamma$ from $X$ with probability $1-\epsilon$ provided both of the following hold:
\begin{align*}
pn_{\mathrm{min}}
&\geq\max\Big\{48(\tfrac{n_{\mathrm{max}}}{n_{\mathrm{min}}})^2\log(\tfrac{36k^3(d+1)}{\epsilon}),16\max\{1,(\tfrac{r}{\Delta/2-r})^2\}\log(\tfrac{2k(d+2)}{\epsilon})\Big\},\\
\operatorname{prox}(X,\Gamma)
&>(\tfrac{2r}{\Delta}+\tfrac{3}{2})\cdot 8r \cdot\sqrt{\tfrac{\log(8k^2(d+2)/\epsilon)}{pn_{\mathrm{min}}}}\\
&\qquad+\tfrac{1}{2}\bigg(\Big(16\sqrt{3}+\sqrt{\tfrac{104}{3}}\Big)\cdot\tfrac{k}{\sqrt{pn_{\mathrm{min}}}}\cdot\tfrac{n_{\mathrm{max}}}{n_{\mathrm{min}}}\cdot r^2\cdot\log^{1/2}(\tfrac{36k^3(d+1)}{\epsilon})\bigg)^{1/2}.
\end{align*}
As we now discuss, there exists an explicit shape parameter $C(X,\Gamma)>0$ such that the inequality $\mathbb{E}|W|\geq C(X,\Gamma)\cdot \log(1/\epsilon)$ implies the above conditions.
Denote 
\[
\pi_{\mathrm{min}}:=\frac{1}{n}\min_{S\in\Gamma}|S|,
\qquad
\pi_{\mathrm{max}}:=\frac{1}{n}\max_{S\in\Gamma}|S|,
\]
and observe that $pn_{\mathrm{min}}=\mathbb{E}|W|\cdot\pi_{\mathrm{min}}$.
This explains the appearance of $\mathbb{E}|W|$ in our desired inequality.
To isolate $\log(1/\epsilon)$, we will use the general observation that $\alpha,\beta\geq\gamma>1$ implies
\[
\log(\alpha\beta)
\leq\frac{2\log(\alpha)\log(\beta)}{\log(\gamma)}.
\]
Indeed, $\frac{\log(\alpha\beta)}{\log(\alpha)\log(\beta)}=\frac{1}{\log(\alpha)}+\frac{1}{\log(\beta)}\leq\frac{2}{\log(\gamma)}$.
We apply this bound several times with $\beta=1/\epsilon$ and $\gamma=2$ so that the following conditions imply the above conditions:
\begin{align*}
\mathbb{E}|W|\cdot\pi_{\mathrm{min}}
&\geq\tfrac{2}{\log(2)}\cdot\log(\tfrac{1}{\epsilon})\cdot\max\Big\{48(\tfrac{\pi_{\mathrm{max}}}{\pi_{\mathrm{min}}})^2\log(36k^3(d+1)),\\
&\hspace{2.7in}16\max\{1,(\tfrac{\Delta}{2r}-1)^{-2}\}\log(2k(d+2))\Big\},
\end{align*}
\begin{align*}
&\operatorname{prox}(X,\Gamma)
>(\tfrac{2r}{\Delta}+\tfrac{3}{2})\cdot 8r \cdot\sqrt{\tfrac{\log(8k^2(d+2))}{\mathbb{E}|W|\cdot\pi_{\mathrm{min}}}\cdot\tfrac{2}{\log(2)}\cdot\log(\tfrac{1}{\epsilon})}\\
&\qquad+\tfrac{1}{2}\bigg(\Big(16\sqrt{3}+\sqrt{\tfrac{104}{3}}\Big)\cdot\tfrac{k}{\sqrt{\mathbb{E}|W|\cdot\pi_{\mathrm{min}}}}\cdot\tfrac{\pi_{\mathrm{max}}}{\pi_{\mathrm{min}}}\cdot r^2\cdot\log^{1/2}(36k^3(d+1))\cdot(\tfrac{2}{\log(2)}\cdot\log(\tfrac{1}{\epsilon}))^{1/2}\bigg)^{1/2}.
\end{align*}
Notice that we may express these conditions in terms of $c:=\mathbb{E}|W|/\log(\tfrac{1}{\epsilon})$:
\begin{align*}
c\cdot\pi_{\mathrm{min}}
&\geq\tfrac{2}{\log(2)}\cdot\max\Big\{48(\tfrac{\pi_{\mathrm{max}}}{\pi_{\mathrm{min}}})^2\log(36k^3(d+1)),16\max\{1,(\tfrac{\Delta}{2r}-1)^{-2}\}\log(2k(d+2))\Big\},\\
\tfrac{\operatorname{prox}(X,\Gamma)}{r}
&>(\tfrac{2r}{\Delta}+\tfrac{3}{2})\cdot 8 \cdot\sqrt{\tfrac{\log(8k^2(d+2))}{c\cdot\pi_{\mathrm{min}}}\cdot\tfrac{2}{\log(2)}}\\
&\quad+\tfrac{1}{2}\bigg(\Big(16\sqrt{3}+\sqrt{\tfrac{104}{3}}\Big)\cdot\tfrac{k}{\sqrt{c\cdot\pi_{\mathrm{min}}}}\cdot\tfrac{\pi_{\mathrm{max}}}{\pi_{\mathrm{min}}}\cdot\log^{1/2}(36k^3(d+1))\cdot(\tfrac{2}{\log(2)})^{1/2}\bigg)^{1/2}.
\end{align*}
The set of $c$ for which the first inequality holds is an interval the form $[c_1,\infty)$, while the set of $c$ for which the second inequality holds takes the form $(c_2,\infty)$.
Then $C(X,\Gamma):=\max\{c_1,c_2\}$ is an explicit function of $d$, $\frac{\Delta}{r}$, $\tfrac{\operatorname{prox}(X,\Gamma)}{r}$, $k$, $\pi_{\mathrm{min}}$, and $\pi_{\mathrm{max}}$, as desired.
\end{proof}

\section{Proof of Theorems~\ref{thm.hoefding monte carlo}(b) and~\ref{thm.markov monte carlo}(b)}
\label{sec.proof thm 5 and 6}

\begin{lemma}[cf.\ Lemma~10 in~\cite{MixonVW:17}]
\label{lem.seminorm bound}
Given a symmetric matrix $M\in\mathbb{R}^{n\times n}$, it holds that
\[
\max_{Z\in\mathcal{Z}(n,k)}|\operatorname{tr}(MZ)|
\leq\min\Big\{\|M\|_*,k\|M\|_{2\to2}\Big\}.
\]
\end{lemma}

\begin{proof}
Select any $Z\in\mathcal{Z}(n,k)$, and let $\alpha_1\geq\cdots\geq\alpha_n$ and $\beta_1\geq\cdots\geq\beta_n$ denote the singular values of $M$ and $Z$, respectively.
Then Von Neumann's trace inequality gives that
\[
|\operatorname{tr}(MZ)|
\leq\sum_{i\in[n]}\alpha_i\beta_i.
\]
Since $Z$ is stochastic, then by the Gershgorin circle theorem, every eigenvalue of $Z$ has modulus at most $1$.
Since $Z$ is symmetric, it follows that $\beta_1\leq1$, and so
\[
|\operatorname{tr}(MZ)|
\leq\sum_{i\in[n]}\alpha_i\beta_i
\leq \beta_1\sum_{i\in[n]}\alpha_i
\leq\|M\|_*.
\]
We also have $\sum_{i\in[n]}\beta_i=\operatorname{tr}Z=k$, and so
\[
|\operatorname{tr}(MZ)|
\leq\sum_{i\in[n]}\alpha_i\beta_i
\leq \alpha_1\sum_{i\in[n]}\beta_i
= k\|M\|_{2\to2}.
\qedhere
\]
\end{proof}

\begin{lemma}
\label{lem.spectral lower bound on sdp}
Given any tuple $X:=\{x_i\}_{i\in[n]}$ of points in $\mathbb{R}^d$ and any orthogonal projection matrix $P\in\mathbb{R}^{d\times d}$, it holds that
\[
\min_{Z\in\mathcal{Z}(n,k)}\frac{1}{2}\operatorname{tr}(D_XZ)
\geq\|PX\|_F^2-k\|PX\|_{2\to2}^2.
\]
(Here, we abuse notation by identifying $X$ with a member of $\mathbb{R}^{d\times n}$.)
\end{lemma}

\begin{proof}
Define $\nu\in\mathbb{R}^n$ to have $i$th coordinate $\|x_i\|^2$.
Then we have
\[
D_X
=\nu 1^\top-2X^\top X+1\nu^\top,
\qquad
\|X\|_F^2=1^\top \nu.
\]
Fix $Z\in\mathcal{Z}(n,k)$.
Then $Z^\top=Z$ and $Z1=1$, and so
\begin{equation}
\label{eq.objective in terms of gram matrix}
\frac{1}{2}\operatorname{tr}(D_XZ)
=\frac{1}{2}\operatorname{tr}((\nu 1^\top-2X^\top X+1\nu^\top)Z)
=\|X\|_F^2-\operatorname{tr}(X^\top XZ).
\end{equation}
We apply Lemma~\ref{lem.seminorm bound} to get
\begin{align*}
\operatorname{tr}(X^\top XZ)
&=\operatorname{tr}(X^\top P XZ)+\operatorname{tr}(X^\top (I-P)XZ)\\
&\leq k\|X^\top PX\|_{2\to2}+\|X^\top (I-P) X\|_*
=k\|PX\|_{2\to2}^2+\|(I-P) X\|_F^2,
\end{align*}
and combining with \eqref{eq.objective in terms of gram matrix} gives
\[
\frac{1}{2}\operatorname{tr}(D_XZ)
\geq \|X\|_F^2-\Big(k\|PX\|_{2\to2}^2+\|(I-P) X\|_F^2\Big)
=\|PX\|_F^2-k\|PX\|_{2\to2}^2.
\qedhere
\]
\end{proof}

\begin{proposition}[Matrix Chernoff inequality, see Theorem~5.1.1 and (5.1.8) in~\cite{Tropp:15}]
\label{prop.matrix chernoff}
Consider a finite sequence $\{X_k\}$ of independent, random, Hermitian matrices with common dimension d, and assume that $0\leq\lambda_{\mathrm{min}}(X_k)\leq\lambda_{\mathrm{max}}(X_k)\leq L$ almost surely for each $k$.
Then
\[
\mathbb{E}\lambda_{\mathrm{max}}\bigg(\sum_kX_k\bigg)
\leq1.72\cdot\lambda_{\mathrm{max}}\bigg(\sum_k\mathbb{E}X_k\bigg)+L\log d.
\]
\end{proposition}

\begin{proposition}[Dvoretzky--Kiefer--Wolfowitz inequality, Theorem~11.5 in~\cite{Kosorok:08}]
\label{prop.dkw inequality}
Consider a sequence $\{X_k\}_{k\in[n]}$ of real-valued independent random variables with common cumulative distribution function $F\colon\mathbb{R}\to[0,1]$, and let $F_n\colon\mathbb{R}\to[0,1]$ denote the random empirical distribution function defined by $F_n(x):=\frac{1}{n}\sum_{k\in[n]}1_{\{X_k\leq x\}}$.
Then for every $\epsilon>0$, it holds that
\[
\mathbb{P}\Big\{\sup_{x\in\mathbb{R}}|F_n(x)-F(x)|>\epsilon\Big\}\leq 2e^{-2n\epsilon^2}.
\]
\end{proposition}

\begin{proposition}[Lemma~1 in~\cite{LaurentM:00}]
\label{prop.lemma1}
Suppose $X$ has chi-squared distribution with $d$ degrees of freedom.
Then for each $t>0$, it holds that
\[
\mathbb{P}\{X\geq d+2\sqrt{dt}+2t\}\leq e^{-t},
\qquad
\mathbb{P}\{X\leq d-2\sqrt{dt}\}\leq e^{-t}.
\]
\end{proposition}

It is sometimes easier to interact with a simpler (though weaker) version of the upper chi-squared tail estimate:

\begin{corollary}
\label{cor.chi squared tail}
Suppose $X$ has chi-squared distribution with $d$ degrees of freedom.
Then $\mathbb{P}\{X\geq x\}\leq e^{-x/5}$ for every $x\geq 5d$.
\end{corollary}

\begin{proof}
Put $t:=x/5\geq d$, and so $d+2\sqrt{dt}+2t\leq t+2t+2t=5t$, and Proposition~\ref{prop.lemma1} gives
\[
\mathbb{P}\{X\geq x\}
=\mathbb{P}\{X\geq 5t\}
\leq\mathbb{P}\{X\geq d+2\sqrt{dt}+2t\}
\leq e^{-t}
=e^{-x/5}.
\qedhere
\]
\end{proof}

\begin{corollary}
\label{cor.bound of max of chi squared}
Suppose $\{X_k\}_{k\in[n]}$ are independent random variables with chi-squared distribution with $d$ degrees of freedom.
Then $\max_{k\in[n]}X_k\leq 2d+6\log n$ with probability $\geq1-\frac{1}{n}$.
\end{corollary}

\begin{proof}
By the arithmetic mean--geometric mean inequality and Proposition~\ref{prop.lemma1}, we have
\[
\mathbb{P}\{X_k\geq 2d+3t\}
\leq\mathbb{P}\{X_k\geq d+2\sqrt{dt}+2t\}
\leq e^{-t}.
\]
Take $t:=2\log n$ and apply a union bound to get the result.
\end{proof}

\begin{proposition}[Corollary~5.35 in~\cite{Vershynin:12}]
\label{prop.gaussian spectral norm bound}
Let $G$ be an $m\times n$ matrix with independent standard Gaussian entries.
Then for every $t\geq0$, it holds that $\|G\|_{2\to2}\leq\sqrt{m}+\sqrt{n}+t$ with probability $\geq1-2e^{-t^2/2}$.
\end{proposition}

Throughout, we let $\mathbb{E}_G$ denote expectation conditioned on $G$, and we assume $n\geq s\geq d\geq k$ without mention.
(These inequalities are typically implied by the hypotheses of our results, but even so, it is helpful to keep these bounds in mind when interpreting the analysis.)

\begin{lemma}
\label{lem.spectral norm of sample with replacement}
Let $G:=\{g_i\}_{i\in[n]}$ denote a tuple of independent random vectors $g_i\sim\mathsf{N}(0,I_d)$, let $S:=\{i_j\}_{j\in[s]}$ denote a tuple of independent random indices $i_j\sim\mathsf{Unif}([n])$, and let $H=H(G,S)\in\mathbb{R}^{d\times s}$ denote the random matrix whose $j$th column is $g_{i_j}$.
Assuming $s\geq 15d\log d$, it holds that $\mathbb{E}_G\|H\|_{2\to2}^2\leq 5s$ with probability $\geq 1-\frac{1}{n}-e^{-\Omega(n/(d+\log n)^2)}$.
\end{lemma}

\begin{proof}
Fix a threshold $\tau>0$ to be selected later, and for any vector $x\in\mathbb{R}^d$, denote
\[
x^-
:=\left\{\begin{array}{cl}
x&\text{if }\|x\|^2\leq\tau\\
0&\text{otherwise}
\end{array}\right\},
\qquad
x^+
:=\left\{\begin{array}{cl}
x&\text{if }\|x\|^2>\tau\\
0&\text{otherwise}
\end{array}\right\}.
\]
Letting $h_j$ denote the $j$th column of $H$, then we are interested in the quantity
\begin{align*}
\|H\|_{2\to2}^2
&=\bigg\|\sum_{j\in[s]}h_jh_j^\top\bigg\|_{2\to2}
\leq\bigg\|\sum_{j\in[s]}h_j^-{h_j^-}^\top\bigg\|_{2\to2}+\bigg\|\sum_{j\in[s]}h_j^+{h_j^+}^\top\bigg\|_{2\to2}.
\end{align*}
We will bound the expectation of the first term using the Matrix Chernoff inequality (Proposition~\ref{prop.matrix chernoff}) and the expectation of the second term using the Dvoretzky--Kiefer--Wolfowitz inequality (Proposition~\ref{prop.dkw inequality}).
First, $\mathbb{E}_G[h_j^-{h_j^-}^\top]=\frac{1}{n}\sum_{i\in[n]}g_i^-{g_i^-}^\top$ for each $j\in[s]$, and so
\[
\lambda_{\mathrm{max}}\bigg(\sum_{j\in[s]}\mathbb{E}_G\Big[h_j^-{h_j^-}^\top\Big]\bigg)
=\lambda_{\mathrm{max}}\bigg(\frac{s}{n}\sum_{i\in[n]}g_i^-{g_i^-}^\top\bigg)
\leq\lambda_{\mathrm{max}}\bigg(\frac{s}{n}\sum_{i\in[n]}g_ig_i^\top\bigg)
=\frac{s}{n}\|G\|_{2\to2}^2,
\]
where the inequality uses the fact that $\frac{s}{n}\sum_{i\in[n]}g_i^+{g_i^+}^\top\succeq0$.
Thus, Matrix Chernoff gives
\[
\mathbb{E}_G\bigg\|\sum_{j\in[s]}h_j^-{h_j^-}^\top\bigg\|_{2\to2}
\leq1.72\cdot\frac{s}{n}\|G\|_{2\to2}^2+\tau\log d.
\]
Next, we bound the expectation of
\[
\bigg\|\sum_{j\in[s]}h_j^+{h_j^+}^\top\bigg\|_{2\to2}
\leq\sum_{j\in[s]}\|h_j^+{h_j^+}^\top\|_{2\to2}
=\sum_{j\in[s]}\|h_j^+\|^2.
\]
Let $F_G^+$ denote the empirical distribution function of $\{\|g_i^+\|^2\}_{i\in[n]}$.
Then
\begin{equation}
\label{eq.integral to bound}
\frac{1}{s}\mathbb{E}_G\bigg\|\sum_{j\in[s]}h_j^+{h_j^+}^\top\bigg\|_{2\to2}
\leq \mathbb{E}_G\|h_1^+\|^2
=\int_0^\infty(1-F_G^+(x))dx.
\end{equation}
Let $F_G$ denote the empirical distribution function of $\{\|g_i\|^2\}_{i\in[n]}$, put $\|G\|_{1\to2}^2:=\max_{i\in[n]}\|g_i\|^2$, and observe that
\[
F_G^+(x)=\left\{\begin{array}{cl}F_G(\tau)&\text{if }x\leq\tau\\ F_G(x)&\text{if }x>\tau\end{array}\right\}
\quad
\text{and}
\quad
F_G(x)=1
\quad
\forall x>\|G\|_{1\to2}^2.
\]
In particular, we may assume $\|G\|_{1\to2}^2\geq\tau$ without loss of generality, since otherwise the integral in \eqref{eq.integral to bound} equals zero.
We will estimate this integral in pieces:
\[
\int_0^\infty
=\int_0^\tau+\int_\tau^{\|G\|_{1\to2}^2}+\int_{\|G\|_{1\to2}^2}^\infty.
\]
We have $\int_{\|G\|_{1\to2}^2}^\infty=0$, and we will estimate the other two integrals in terms of the quantity
\[
E(G):=\sup_{x\in\mathbb{R}}|F_G(x)-F(x)|,
\]
where $F$ denotes the cumulative distribution function of the chi-squared distribution with $d$ degrees of freedom.
(Later, we will apply the Dvoretzky--Kiefer--Wolfowitz inequality to bound $E(G)$ with high probability on $G$.)
First, we have
\[
\int_0^\tau(1-F_G^+(x))dx
=\tau(1-F_G(\tau))
\leq\tau(1-F(\tau)+E(G))
\leq\tau(e^{-\tau/5}+E(G)),
\]
where the last step applies Corollary~\ref{cor.chi squared tail} under the assumption that $\tau\geq5d$.
Next, for each $x\in(\tau,\|G\|_{1\to2}^2)$, we have 
\[
1-F_G^+(x)
=1-F_G(x)
\leq 1-F(x)+E(G),
\]
and so integrating gives
\[
\int_\tau^{\|G\|_{1\to2}^2}(1-F_G^+(x))dx
\leq\int_\tau^\infty(1-F(x))dx+\|G\|_{1\to2}^2\cdot E(G).
\]
We estimate the first term using Corollary~\ref{cor.chi squared tail}:
\[
\int_\tau^\infty(1-F(x))dx
\leq\int_\tau^\infty e^{-x/5}dx
=5e^{-\tau/5}.
\]
All together, we have
\[
\mathbb{E}_G\|H\|_{2\to2}^2
\leq 1.72\cdot\frac{s}{n}\|G\|_{2\to2}^2+\tau\log d
+s\Big(\tau(e^{-\tau/5}+E(G))+5e^{-\tau/5}+\|G\|_{1\to2}^2\cdot E(G)\Big).
\]
It remains to bound this random variable in a high-probability event.
Proposition~\ref{prop.gaussian spectral norm bound} gives $\|G\|_{2\to2}^2=O(n)$ with high probability when $n\geq d$.
This means the first term in our bound will be $O(s)$, and so we select $\tau:=s/\log d$ so that the second term has the same order.
(Note that our assumption $\tau\geq 5d$ then requires $s\geq 5d\log d$.)
The remaining term is $s$ times
\begin{equation}
\label{eq.want to be O(1)}
\tau(e^{-\tau/5}+E(G))+5e^{-\tau/5}+\|G\|_{1\to2}^2\cdot E(G)
\leq (\tau+5)e^{-\tau/5}+2\|G\|_{1\to2}^2\cdot E(G),
\end{equation}
where the inequality follows from the bound $\tau\leq\|G\|_{1\to2}^2$.
We want \eqref{eq.want to be O(1)} to be $O(1)$.
The first term in \eqref{eq.want to be O(1)} is smaller than $1$ provided $\tau\geq15$.
For the second term in \eqref{eq.want to be O(1)}, we recall from Corollary~\ref{cor.bound of max of chi squared} that $\|G\|_{1\to2}^2=O(d+\log n)$ with high probability.
This suggests that we restrict to an event in which $E(G)=O(1/(d+\log n))$, which we can estimate using the Dvoretzky--Kiefer--Wolfowitz inequality.
For this bound, the failure probability will be $2e^{-2n\epsilon^2}=\exp(-\Omega(n/(d+\log n)^2))$, and this informs how sharply we can bound $\|G\|_{2\to2}^2$.
Overall, we restrict to an event in which three things occur simultaneously:
\[
\|G\|_{2\to2}^2\leq 1.1n,
\qquad
\|G\|_{1\to2}^2\leq 2d+6\log n, 
\qquad
E(G)\leq\frac{1}{4d+12\log n}.
\]
A union bound gives that this event has probability $\geq 1-\frac{1}{n}-e^{-\Omega(n/(d+\log n)^2)}$, and over this event, it holds that $\mathbb{E}_G\|H\|_{2\to2}^2\leq 1.72\cdot 1.1s+s+s(1+1)\leq 5s$.
\end{proof}

\begin{lemma}
\label{lem.bounds on E_ X SDP}
Draw $X:=\{x_i\}_{i\in[n]}$ in $\mathbb{R}^d$ from a mixture of $k$ gaussians with equal weights and identity covariance.
Explicitly, take any $\mu_1,\ldots,\mu_k\in\mathbb{R}^d$, draw $\{t_i\}_{i\in[n]}$ independently with distribution $\mathsf{Unif}([k])$, draw $\{g_i\}_{i\in[n]}$ independently with distribution $\mathsf{N}(0,I_d)$, and take $x_i:=\mu_{t_i}+g_i$.
Next, draw $\{i_j\}_{j\in[s]}$ independently with distribution $\mathsf{Unif}([n])$ and define the random tuple $Y:=\{y_j\}_{j\in[s]}$ by $y_j:=x_{i_j}$.
Then provided $s\geq 15d\log d$, it holds that
\[
d-6k-1
\leq\mathbb{E}_X\operatorname{SDP}(Y,k)
\leq\operatorname{IP}(X,k)
\leq d+1
\]
with probability $\geq 1-\frac{1}{n}-e^{-\Omega(n/(d+\log n)^2)}$.
\end{lemma}

\begin{proof}
Lemma~\ref{lem.ip sketch bound} gives the middle inequality in our claim.
Next, we obtain an upper bound on $\operatorname{IP}(X,k)$ by passing to the partition $\Gamma\in\Pi(n,k)$ defined by the level sets of the planted assignment $i\mapsto t_i$:
\[
\operatorname{IP}(X,k)
\leq\frac{1}{n}\sum_{S\in\Gamma}\sum_{i\in S}\bigg\|x_i-\frac{1}{|S|}\sum_{j\in S}x_j\bigg\|^2
\leq\frac{1}{n}\sum_{S\in\Gamma}\sum_{i\in S}\|x_i-\mu_{t_i}\|^2
=\frac{1}{n}\sum_{i\in[n]}\|g_i\|^2,
\]
where the second inequality follows from the fact that the centroid of a tuple of points minimizes the sum of squared distances from those points.
Next, $\sum_{i\in[n]}\|g_i\|^2$ has chi-squared distribution with $dn$ degrees of freedom, and so an application of Proposition~\ref{prop.lemma1} gives that $\operatorname{IP}(X,k)\leq d+1$ with probability $\geq1-e^{-n/(16d)}$.
It remains find a lower bound on $\mathbb{E}_X\operatorname{SDP}(Y,k)$.
For this, we will apply Lemma~\ref{lem.spectral lower bound on sdp} with $P$ representing the orthogonal projection map onto $(\operatorname{span}\{\mu_t\}_{t\in[k]})^\perp$.
Letting $H:=\{h_j\}_{j\in[s]}$ denote the random tuple defined by $h_j:=g_{i_j}$, which satisfies $Ph_j=Pg_{i_j}=Px_{i_j}=Py_j$, we have
\[
\operatorname{SDP}(Y,k)
\geq\frac{1}{s}\Big(\|PY\|_F^2-k\|PY\|_{2\to2}^2\Big)
\geq\frac{1}{s}\Big(\|PH\|_F^2-k\|H\|_{2\to2}^2\Big).
\]
Letting $G\in\mathbb{R}^{d\times n}$ denote the matrix whose $i$th column is $g_i$, we take expectations of both sides to get
\begin{equation}
\label{eq.exp sdp to lower bound}
\mathbb{E}_X\operatorname{SDP}(Y,k)
\geq\frac{1}{s}\mathbb{E}_G\|PH\|_F^2-\frac{k}{s}\mathbb{E}_G\|H\|_{2\to2}^2.
\end{equation}
The first term can be rewritten as
\[
\frac{1}{s}\mathbb{E}_G\|PH\|_F^2
=\frac{1}{s}\mathbb{E}_G\sum_{j\in[s]}\|Ph_j\|^2
=\frac{1}{s}\sum_{j\in[s]}\mathbb{E}_G\|Ph_j\|^2
=\frac{1}{n}\sum_{i\in[n]}\|Pg_i\|^2
=\frac{1}{n}\|PG\|_F^2.
\]
Furthermore, $\|PG\|_F^2$ has chi-squared distribution with $(d-k')n$ degrees of freedom, where $k'$ denotes the dimension of $\operatorname{span}\{\mu_t\}_{t\in[k]}$.
Thus, Proposition~\ref{prop.lemma1} gives $\frac{1}{n}\|PG\|_F^2\geq d-k'-1$ with probability $\geq1-e^{-n/(4d)}$.
The second term in \eqref{eq.exp sdp to lower bound} is bounded by Lemma~\ref{lem.spectral norm of sample with replacement}, and the result follows from a union bound.
\end{proof}

\begin{proof}[Proof of Theorem~\ref{thm.hoefding monte carlo}(b)]
Recall that Lemma~\ref{lem.bounds on E_ X SDP} gives
\[
d-6k-1
\leq\mathbb{E}_X\operatorname{SDP}(Y,k)
\leq\operatorname{IP}(X,k)
\leq d+1
\]
with probability $\geq 1-\frac{1}{n}-e^{-\Omega(n/(d+\log n)^2)}$.
We will use Hoeffding's inequality to show that
\[
B_H
\geq \mathbb{E}_X\operatorname{SDP}(Y,k)-1
\]
with probability $\geq1-\epsilon$, from which the result follows by a union bound.
Our use of Hoeffding's inequality requires a bound on $b$.
To this end, Lemma~\ref{lem.deterministic kmeans++ approximation ratio} gives that
\[
b
\leq 4\operatorname{inf}(f)^2
\leq 4f(\{\mu_t\}_{t\in[k]})^2,
\]
which in turn is at most $4$ times the maximum of $n$ independent chi-squared random variables with $d$ degrees of freedom.
Corollary~\ref{cor.bound of max of chi squared} bounds this quantity by $4(2d+6\log n)$, and this bound holds in the event considered in Lemma~\ref{lem.bounds on E_ X SDP}.
Combining the bounds $b\leq4(2d+6\log n)$ and $\ell\geq128(d+3\log n)^2\log(1/\epsilon)$ then gives $t:=(\tfrac{b^2}{2\ell}\log(\tfrac{1}{\epsilon}))^{1/2}
\leq\frac{1}{2}$, and so $t\leq1-t$.
Hoeffding's inequality then gives
\begin{align*}
\mathbb{P}\{B_H < \mathbb{E}_X\operatorname{SDP}(Y,k)-1\}
&=\mathbb{P}\bigg\{\frac{1}{\ell}\sum_{i\in[\ell]}\operatorname{SDP}(Y_i,k)-t < \mathbb{E}_X\operatorname{SDP}(Y,k)-1\bigg\}\\
&\leq\exp(-\tfrac{2\ell(1-t)^2}{b^2})
\leq\exp(-\tfrac{2\ell t^2}{b^2})
=\epsilon,
\end{align*}
where the last step applies the definition of $t$.
\end{proof}

\begin{proof}[Proof of Theorem~\ref{thm.markov monte carlo}(b)]
We have from the proof of the upper bound in Lemma~\ref{lem.bounds on E_ X SDP} that $\operatorname{IP}(X,k)\leq d+1$ with probability $\geq1-e^{-n/(16d)}$.
Thus, it suffices to show that $B_M\geq d-3k-2$ with probability $\geq1-e^{-s/(8d)}-2e^{-s/54}$.
To this end, we first observe that
\begin{equation}
\label{eq.simple bounds for B_M}
d-3k-2
\leq (1-\tfrac{1}{d})(d-3k-1)
\qquad
\text{and}
\qquad
\epsilon^{1/\ell}
\geq 1-\tfrac{1}{d}.
\end{equation}
Indeed, the first inequality can be seen by expanding, while the second inequality follows from the assumption $\ell\geq d\log(1/\epsilon)$.
In particular, $d\geq1$ implies $\frac{1}{d}\leq\log(\frac{1}{1-1/d})$, and so
\[
\ell
\geq d\log(1/\epsilon)
\geq \frac{\log(1/\epsilon)}{\log(\frac{1}{1-1/d})},
\]
and rearranging gives the desired inequality.
We apply \eqref{eq.simple bounds for B_M} with a union bound to get
\begin{align*}
\mathbb{P}\Big\{B_M<d-3k-2\Big\}
&\leq\mathbb{P}\Big\{\epsilon^{1/\ell}\min_{i\in[\ell]}\operatorname{SDP}(Y_i,k)<(1-\tfrac{1}{d})(d-3k-1)\Big\}\\
&\leq\mathbb{P}\Big\{\min_{i\in[\ell]}\operatorname{SDP}(Y_i,k)<d-3k-1\Big\}\\
&\leq\ell\cdot\mathbb{P}\Big\{\operatorname{SDP}(Y,k)<d-3k-1\Big\},
\end{align*}
where $Y$ is a random matrix with the same distribution as each $Y_i$.
In particular, the columns of $Y$ are drawn uniformly without replacement from $X$.
Let $S:=\{i_j\}_{j\in[s]}$ denote the random indices such that the $j$th column of $Y$ is $y_j=x_{i_j}$.
By the law of total probability, it suffices to bound the conditional probability
\[
\mathbb{P}_S\Big\{\operatorname{SDP}(Y,k)<d-3k-1\Big\}
\]
uniformly over all possible realizations of $S$.
Recall that $x_i=\mu_{t_i}+g_i$, where $t_i\sim\mathsf{Unif}([k])$ and $g_i\sim\mathsf{N}(0,I_d)$.
Let $H\in\mathbb{R}^{d\times s}$ denote the random matrix whose $j$th column is $h_j:=g_{i_j}$.
Similar to the proof of the lower bound in Lemma~\ref{lem.bounds on E_ X SDP}, we take $P$ to be the orthogonal projection onto the $(d-k')$-dimensional subspace $(\operatorname{span}\{\mu_t\}_{t\in[k]})^\perp$ and then apply Lemma~\ref{lem.spectral lower bound on sdp} to get
\[
\operatorname{SDP}(Y,k)
\geq\frac{1}{s}\Big(\|PH\|_F^2-k\|H\|_{2\to2}^2\Big).
\]
This implies
\begin{align}
\nonumber
\mathbb{P}_S\Big\{\operatorname{SDP}(Y,k)<d-3k-1\Big\}
&\leq \mathbb{P}_S\Big\{\tfrac{1}{s}\big(\|PH\|_F^2-k\|H\|_{2\to2}^2\big)<d-3k-1\Big\}\\
\label{eq.conditional probabilities to bound}
&\leq \mathbb{P}_S\Big\{\tfrac{1}{s}\|PH\|_F^2<d-k-1\Big\}+ \mathbb{P}_S\Big\{\tfrac{k}{s}\|H\|_{2\to2}^2>2k\Big\}.
\end{align}
Conditioned on $S$, the entries of $H$ are independent with standard gaussian distribution.
As such, $\|PH\|_F^2$ has chi-squared distribution with $(d-k')s$ degrees of freedom, and so we apply the second part of Proposition~\ref{prop.lemma1} with $t=s/(4(d-k'))$ to bound the first term in \eqref{eq.conditional probabilities to bound} by $e^{-s/(4d)}$.
Next, we apply Proposition~\ref{prop.gaussian spectral norm bound} with $t=(\sqrt{2}-\sqrt{d/s}-1)\sqrt{s}$ to bound the second term in \eqref{eq.conditional probabilities to bound} by $2e^{-t^2/2}$, which in turn is at most $2e^{-s/27}$ since $s\geq54d\log\ell\geq54d$ by assumption.
Overall, we have
\begin{align*}
\mathbb{P}\Big\{B_M<d-3k-2\Big\}
&\leq\ell\cdot\mathbb{P}\Big\{\operatorname{SDP}(Y,k)<d-3k-1\Big\}\\
&=\ell\cdot\mathbb{E}\Big[\mathbb{P}_S\Big\{\operatorname{SDP}(Y,k)<d-3k-1\Big\}\Big]\\
&\leq e^{s/(54d)}\cdot(e^{-s/(4d)}+2e^{-s/27})
\leq e^{-s/(8d)}+2e^{-s/54},
\end{align*}
as desired.
\end{proof}

\section*{Acknowledgments}

Part of this research was conducted while SV was a Research Fellow at the Simons Institute for Computing, University of California at Berkeley.
DGM was partially supported by NSF DMS 1829955, AFOSR FA9550-18-1-0107, and an AFOSR Young Investigator Research Program award.
SV is partially supported by ONR N00014-22-1-2126, NSF CISE 2212457, an AI2AI Amazon research award, and the NSF–Simons Research Collaboration on the Mathematical and Scientific Foundations of Deep Learning (MoDL) (NSF DMS 2031985).

\end{document}